\documentclass[journal]{IEEEtran}

\usepackage{latexsym,amssymb,amsmath,graphicx,epsf,cite,bbm}
\usepackage{epsfig}
\usepackage{graphicx}
\usepackage{algorithm,algorithmic}
\usepackage{amsmath}
\usepackage{amsthm}
\usepackage{slashbox}
\usepackage{pict2e}
\usepackage{subcaption}

\newtheorem{corollary}{Corollary}

\newcommand{\be}{\begin{equation}}
\newcommand{\ee}{\end{equation}}
\newcommand{\bea}{\begin{eqnarray}}
\newcommand{\eea}{\end{eqnarray}}

\renewcommand{\vec}[1]{ \mbox{$\mathbf {#1}$}}
\renewcommand{\baselinestretch}{0.99}

\newcommand{\ei}{\end{itemize}}
\newcommand{\bi}{\begin{itemize}}

\newcommand{\MB}{\left[\begin{array}}
\newcommand{\ME}{\end{array}\right]}

\newtheorem{theorem}{Theorem}

\newtheorem{proposition}{Proposition}

\begin{document}
\title{An Asymptotically Optimal Contextual Bandit Algorithm Using Hierarchical Structures}
\author{ Mohammadreza Mohaghegh Neyshabouri, Kaan Gokcesu, Huseyin Ozkan, and Suleyman S. Kozat, {\em Senior Member, IEEE} \thanks{This work is supported in part by Turkish Academy of Sciences Outstanding Researcher Programme, TUBITAK Contract No. 113E517.

M. Mohaghegh N. and S. S. Kozat are with the Department of Electrical and Electronics Engineering, Bilkent University, Ankara, Turkey, e-mail: \{mohammadreza, kozat\}@ee.bilkent.edu.tr, tel: +90 (312) 290-2336. \par 
K. Gokcesu is with the Department of Electrical Engineering and Computer Science, Massachusetts Institute of Technology, Cambridge, MA 02139 USA, e-mail: gokcesu@mit.edu. \par 
H. Ozkan is with the Faculty of Engineering and Natural Sciences at Sabanc{\i} University, Istanbul 34956 Turkey, e-mail: hozkan@sabanciuniv.edu, tel: +90 (216) 483-9594.}}
\maketitle
\begin{abstract}
We propose online algorithms for sequential learning in the contextual multi-armed bandit setting. Our approach is to partition the context space and then optimally combine all of the possible mappings between the partition regions and the set of bandit arms in a data driven manner. We show that in our approach, the best mapping is able to approximate the best arm selection policy to any desired degree under mild Lipschitz conditions. Therefore, we design our algorithms based on the optimal adaptive combination and asymptotically achieve the performance of the best mapping as well as the best arm selection policy. This optimality is also guaranteed to hold even in adversarial environments since we do not rely on any statistical assumptions regarding the contexts or the loss of the bandit arms. Moreover, we design efficient implementations for our algorithms in various hierarchical partitioning structures such as lexicographical or arbitrary position splitting and binary trees (and several other partitioning examples). For instance, in the case of binary tree partitioning, the computational complexity is only log-linear in the number of regions in the finest partition. In conclusion, we provide significant performance improvements by introducing upper bounds (w.r.t. the best arm selection policy) that are mathematically proven to vanish in the average loss per round sense at a faster rate compared to the state-of-the-art. Our experimental work extensively covers various scenarios ranging from bandit settings to multi-class classification with real and synthetic data. In these experiments, we show that our algorithms are highly superior over the state-of-the-art techniques while maintaining the introduced mathematical guarantees and a computationally decent scalability.
\end{abstract}

\section{Introduction}
We study online learning \cite{JLin,LJian} in the contextual multi-armed bandit setting \cite{ARakotomamonjy,Jpeng,GDitzler,ref1,OnlineLearning,CesaBianchi}. In the classical formulation of the multi-armed bandit problem, one of the available $M$ bandit arms (or {\em actions}) is chosen at each round to obtain a reward (or loss), and the reward (or loss) of all of the other unchosen $M-1$ arms stay oblivious. The objective is to maximize the cumulative reward of the selected arms in a series of rounds. Since the reward we would obtain from the other arms remain hidden, this setting can be considered as a limited feedback version of prediction with expert advice \cite{SEYuksel,HozkanTNNLS,expert1,expert2,expert3,expert4}. Additionally, the well-known fundamental trade-off between exploration and exploitation \cite{TMannucci, Cesa} naturally appears in multi-armed bandits. One should balance exploitation of actions that gave the highest payoffs in the past and exploration of actions that might give higher payoffs in the future.\par 

The multi-armed bandit problem has attracted significant attention due to the applicability of the bandit setting in a wide range of applications from online advertisement \cite{bandit3} and recommender systems \cite{tekin,Tang:2014:ECB:2645710.2645732,XLuo} to clinical trials \cite{clinical} and cognitive radio \cite{cogra1,cogra2}. For example, in the online advertisement application, different ads available to display to users are modeled as the bandit arms and the act of clicking by the user on the displayed ad is modeled as the reward \cite{bandit3}. \par 

In many instances of the bandit algorithms, additional information is available \cite{Lu} such as the age or the gender of the patient in clinical trials \cite{clinicex}, which is useful about the arm selection decision. However, most of the conventional bandit algorithms do not exploit or fail to fully exploit this information\cite{bandit1,bandit4,bandit5}. To remedy, contextual multi-armed bandit algorithms are introduced \cite{bandit2,bandit3,Cesa}, where the additional information is represented as a context vector.
For example, in the online advertisement applications, this context vector may contain certain information about the users such as historical activities or demographic/geographical information. Then the goal of the multi-armed bandit problem is extended to maximally exploit this additional information, i.e., the context, for optimizing the arm selection strategy and therefore gaining more rewards (or suffering less loss).\par 

We consider the contextual extension in the online setting, where we operate sequentially on a stream of observations from a possibly non-stationary, chaotic or even adversarial environment \cite{SRotaBulo, TangAd, Auerthenonstochastic}. Hence, we have no statistical assumptions on the context vectors and behavior of the bandit arms so that our results are guaranteed to hold in an individual sequence manner \cite{Cesa}. We follow a competitive algorithm perspective \cite{Cesa} and define the performance (total time accumulated reward or loss) with respect to a competition class of context dependent bandit arm selection policies.
For this purpose, we design an exponentially large and parameterized competition class of predetermined mappings from the space of context vectors to the bandit arms such that the best arm selection policy\footnote{This best arm selection policy is based on the fixed best partitioning of the context space and the best assignment of the arms to the regions of that best partition. It is not necessarily in our competition class. However, it can be approximated arbitrarily well by the optimal mapping in the class by varying the class parameter; and it can be determined only when the complete data stream is observed.} can be approximated arbitrarily well to a desired degree by the optimal mapping in the competition class.
We point out that each mapping in our competition class partitions the space of context vectors into several disjoint regions and assigns each one of these regions to one of the bandit arms, i.e., each mapping selects the bandit arm corresponding to the region containing the observed context vector. 
Based on this competition class of such mappings, our goal is to asymptotically -at least- achieve\footnote{In addition to achieving, we might well outperform since our approach is data driven and based on combination of partitions, i.e., we do not rely on a single fixed partition.} the performance of the optimal mapping as well as the performance of the best arm selection policy at a faster convergence (performance-wise or in terms of the convergence of the regret upper bound to zero) rate compared to the state-of-the-art as more data is observed.\par

In order to generate partitions of the context space and therefore a rich competition class, we use various hierarchical partitioning structures \cite{ContextWeighting} such as the ones based on lexicographical or arbitrary position splitting, binary trees and several other partitioning examples, cf. Section \ref{BT}. In our design, each of these structures leads to a different competition class but approximates (arbitrarily well, and even perfectly if desired) the same best arm selection policy by the optimal mapping in the corresponding competition class. However, each hierarchical structure encodes the best arm selection policy differently and one of them is the most efficient in the sense of the required number of partition regions (i.e. less number of regions means higher efficiency). Therefore, we explore various hierarchical structures and introduce algorithms for each of such structures by using a carefully designed weighting over the corresponding competition class. The output of the introduced algorithms is the optimal data adaptive combination (w.r.t. the designed weighting) of the policies (aforementioned mappings) in the competition class. Our weighting/adaptive combination favors simpler models in the beginning of the data stream and gradually switches to more complex ones as the data overwhelms. \par 

As a result, our algorithms are guaranteed to asymptotically perform -at least- as well as the best arm selection policy. We achieve this performance optimality at a faster convergence rate (for instance, at the rate $O(\sqrt{(R M \ln{M} \ln{N})/T})$ in the case of binary tree partitioning after averaging the regret bound over $T$ where $R$ is the number of regions in the optimal partition, $M$ is the number of bandit arms, $N$ is the number of regions in the finest partition in the competition class and $T$ is the number of rounds) compared to the state-of-the-art\footnote{The convergence rates given here samples our general regret results (after averaging over $T$) in the case of binary tree partitioning. Our rates for other partitionings in our generic class of hierarchical structures naturally vary but our superiority compared to the state-of-the-art stays valid in a similar manner, cf. Section \ref{BT} for our complete regret results for all structures.} rate $O(\sqrt{(M N\ln{M})/T})$. Note that here, typically, $N>>R$ is the dominating factor. Our superior performance is due to exploiting the right hierarchical partitioning structure that encodes the best policy more efficiently and therefore assigns higher initial weights to the optimal partition. This exploitation of the right structure with the introduced weighting scheme also mitigates the overfitting issue as an additional merit. \par 

We emphasize that our algorithms are designed to work for a generic class of hierarchical partitioning structures and our optimality results do hold for each type of structure in this generic class. Therefore, one can use the proposed algorithms with any type of partitioning that is appropriate for the target application with the corresponding performance guarantees. Such guarantees include upper bounds on the regret w.r.t. the best arm selection policy that are mathematically proven to vanish at $O(1/\sqrt{T})$ (after averaging over $T$) in a superior manner over the state-of-the-art, cf. the following Section \ref{sec:PA} {\em Prior Art} and Section \ref{BT} for detailed comparisons. We also present computationally highly efficient implementations for the introduced algorithms that, for instance, combine $M^N$ mappings with only computational complexity of $O(M\ln{N})$ in the case of binary tree partitioning structure. Through an extensive set of experiments with real and synthetic data, we demonstrate the proposed approach in several scenarios such as multi-class classification, online advirtisement and multi-armed bandit along with various partitioning structures. In these experiments, our algorithms are shown to significantly outperform the state-of-the-art techniques with real-time data processing and strong modeling capabilities.

\subsection{Prior Art} \label{sec:PA}
The contextual bandit problem is mostly studied in the stochastic setting \cite{bandit2,hsu2014taming,Dudik11}, where context vectors and losses are assumed to be drawn randomly and independently from an unknown distribution. Additional assumptions regarding the relations between the context vectors and the arm losses are also used in other studies, e.g., a linear relation in \cite{bandit3} and \cite{auer2002using}, and more general ones in \cite{agarwal2012contextual}. These algorithms essentially fail to hold their performance guarantees if the context vectors or the arm losses are chosen by an adversary rather than a prefixed distribution. \par

An alternative to the stochastic approaches is the adversarial setting, where algorithms do not use any assumptions on the behavior of the context vectors and bandit arms. The well-known \textit{EXP3} algorithm \cite{Auerthenonstochastic} formulates the non-contextual bandit problem in an adversarial setting and achieves a regret upper bound\footnote{We illustrate regret upper bounds  without averaging over $T$ here in this section; but with averaging in the previous section to demonstrate the convergence to $0$ there.} of ${O}(\sqrt{TM \ln{M}})$ against the best arm. \textit{S-EXP3} algorithm \cite{Cesa} is a naive extension of \textit{EXP3} in the contextual  setting, which partitions the context space and runs independent \textit{EXP3} algorithms over each one of the partition regions.  \textit{S-EXP3} achieves a regret upper bound of ${O}(\sqrt{TNM \ln{M}})$ against the best mapping from the regions to the bandit arms, where $N$ is the number of regions in the partition of the context space.
As implied by the regret bound, the \textit{S-EXP3} algorithm works well only when the complexity (the granularity or the level of detailing/fine-ness) of the required partitioning to model the truly optimal selection policy is relatively small, otherwise it quickly overfits and suffer from insufficient data.\par 
\begin{figure*}[t]
	\begin{subfigure}{0.25\textwidth}
		\includegraphics[width=\linewidth]{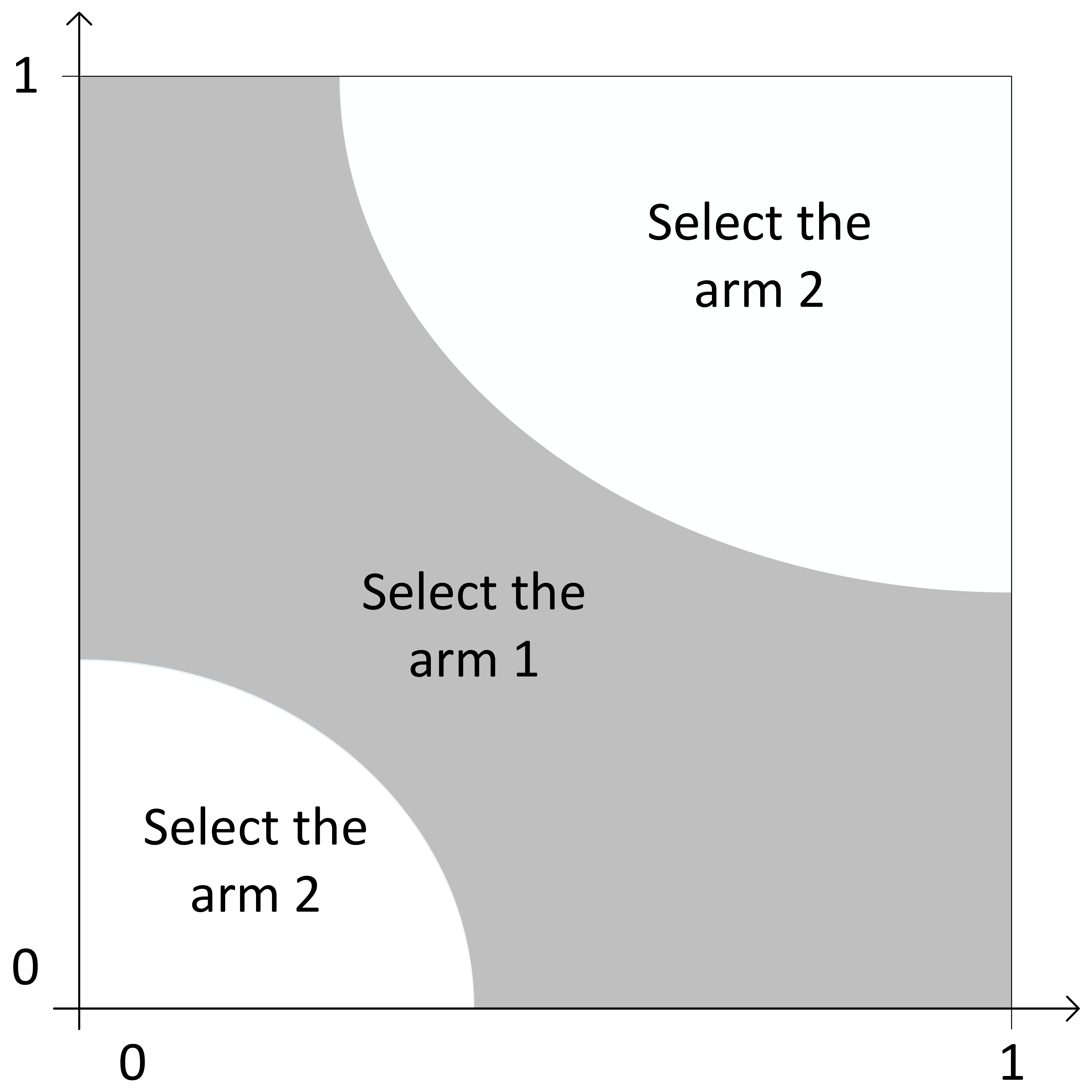}
		\caption{An example mapping from the context space $[0,1]^2$ to the set of bandit arms $\lbrace  1,2 \rbrace$.~~~~~~~~~~~} \label{fig:1a}
	\end{subfigure}
	\hspace*{\fill} 
	\begin{subfigure}{0.25\textwidth}
		\includegraphics[width=\linewidth]{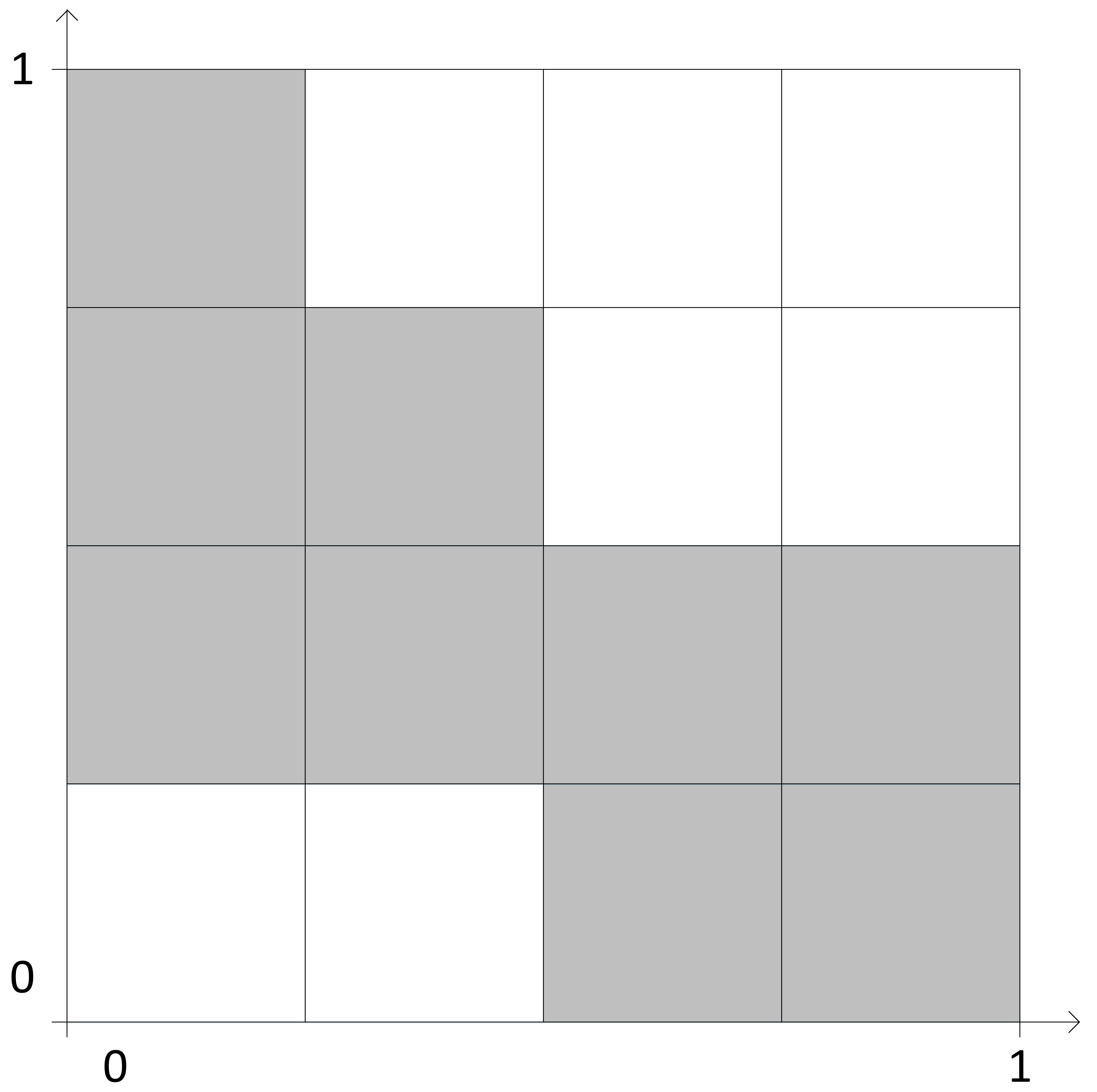}
		\caption{Closest mapping in the quantized competition class with $16$ quantization levels to the mapping in Fig. \ref{fig:1a}.} \label{fig:1b}
	\end{subfigure}
	\hspace*{\fill} 
	\begin{subfigure}{0.25\textwidth}
		\includegraphics[width=\linewidth]{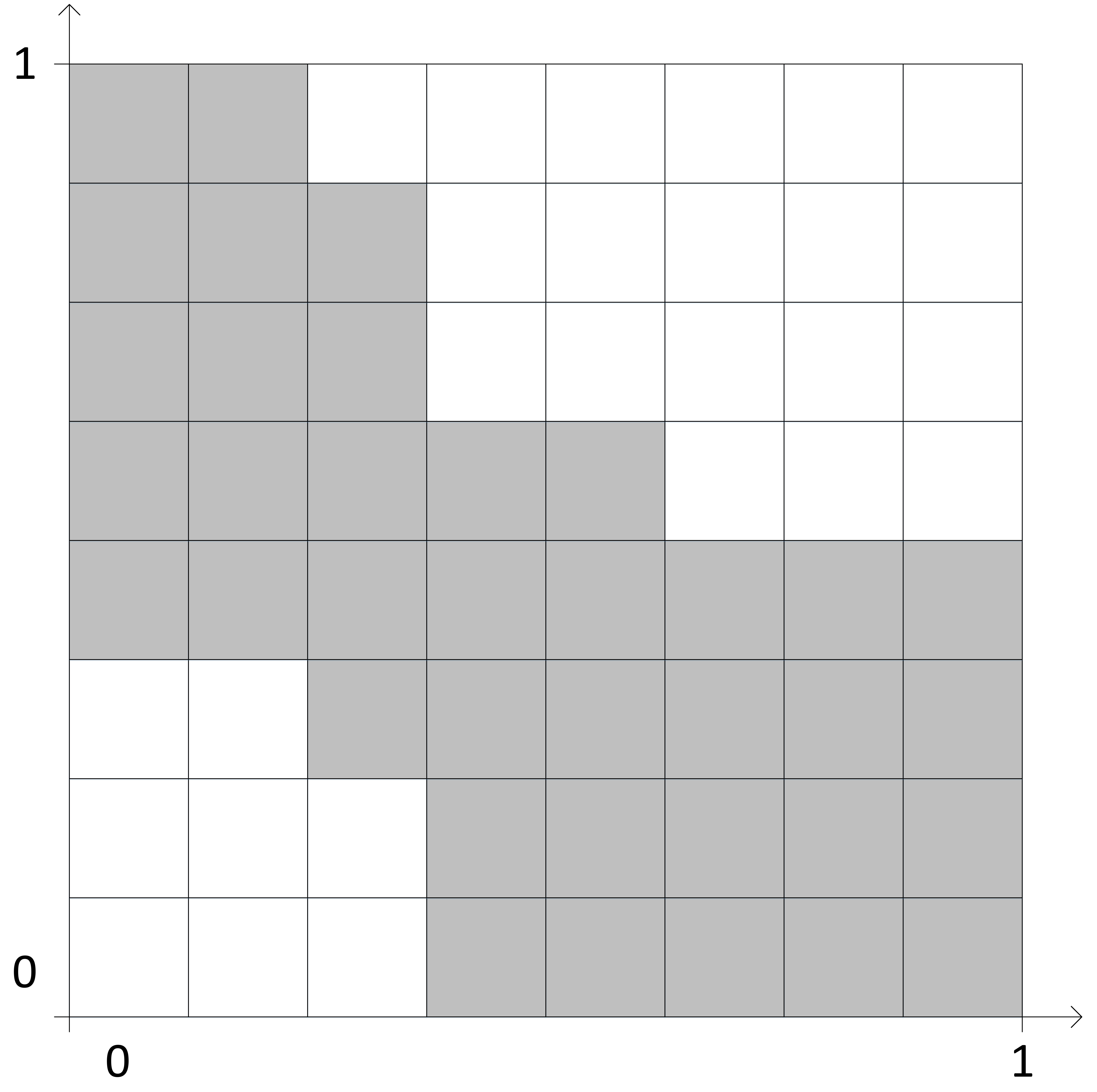}
		\caption{Closest mapping in the quantized competition class with $64$ quantization levels to the mapping in Fig. \ref{fig:1a}.} \label{fig:1c}
	\end{subfigure}
	\caption{An example mapping from the context space to the set of bandit arms and its approximations in the quantized competition classes. In each mapping above, the dark and bright sections are mapped to the arms 1 and 2, respectively.} 
	\label{fig:1}
\end{figure*}
The \textit{EXP4} algorithm \cite{Auerthenonstochastic} is another extension of \textit{EXP3} in the contextual setting. In this algorithm, a set of $K$ experts observe the context vectors and suggest distributions on the arms. Their suggestions are adaptively combined to select the arm to pull. It is shown that \textit{EXP4} achieves a regret upper bound of ${O}(\sqrt{TM \ln K})$ against the best expert. Considering the $M^N$ mappings from a partition of the context space to the arms as the $K$ experts, \textit{EXP4} achieves ${O}(\sqrt{TNM\ln{M}})$ against the optimal mapping. As we show in Section \ref{Algorithm}, the \textit{EXP4} algorithm can be improved by producing an initial tendency (in earlier times of the stream) toward the mappings of smaller complexity. In this case, although the finest partition has $N$ regions (and hence there are $M^N$ mappings in total), it suffices to run \textit{EXP4} over ${O}((NM)^R)$ mappings with $R$ regions resulting a regret bound of ${O}(\sqrt{TMR\ln{(NM)}})$, if the optimal partition consists of $R$ regions.  However, the main problem with this algorithm is its computational complexity of ${O}((NM)^R)$. On the other hand, the \textit{CSB-FTPL} algorithm \cite{effalg} achieves a regret upper bound of ${O}(T^{2/3}M\sqrt{\ln{K}})$ against the best expert among a set of $K$ experts with a computational complexity that is polynomial in $\ln{K}$. Hence, running \textit{CSB-FTPL} over ${O}((NM)^R)$ mappings with $R$ disjoint regions yields a regret upper bound of ${O}(T^{2/3}M\sqrt{R\ln{N}})$ with a polynomial computational complexity in $\ln{N}$. \par 

We emphasize that we seek to achieve a regret upper bound vanishing (w.r.t. rounds/time after averaging over $T$) faster than that of \textit{EXP4} with a computational complexity linear in $\ln{N}$ which allows us to grow the hierarchical structure freely. To this end, our algorithms not only drastically reduce the computational complexity (e.g., down to $O(M\ln{N})$ in the case of binary tree partitioning) compared to the discussed state-of-the-art techniques, but also achieves a regret upper bound of ${O}(\sqrt{TMR\ln{M} \ln{N}})$.\par 

Finally, a simple instance of our hierarchical structures, the context trees, are widely used in various applications including but not limited to data compression \cite{willems1995context,sadakane2000implementing}, estimation \cite{csiszar2006context,6780620}, communications \cite{babich1999context}, regression \cite{kozat2007universal,vanli2014comprehensive} and classification \cite{HOzkanTSP}. In all aforementioned applications, context trees are used to partition the context space in a nested structure, run an independent adaptive model over each one of the tree nodes and combine the models. 
On the other hand, in this paper, we use a generalized novel notion of hierarchical structures that is specifically designed for the completely different multi-armed contextual bandit problem.

\subsection{Contributions}
\begin{itemize}
	\item We introduce novel and efficient contextual bandit arm selection algorithms, which first quantize the space of context vectors and then achieve the performance of the optimal mapping from the quantized regions to the bandit arms (in the average loss per round sense).
	\item We introduce an efficient quantization method and show that using this quantization method, our algorithms asymptotically achieve (not only the optimal mapping but also) the performance of the best arm selection policy (in the average loss per round sense) as the number of quantization levels increases.
	\item We introduce a novel and generalized notion of hierarchical context space partitioning structures for the contextual bandit setting and use such hierarchical structures to design efficient implementations of our algorithms and achieve a faster convergence rate for the regret compared to the state-of-the-art.
	\item We demonstrate significant performance gains with the proposed algorithms in comparison to the state-of-the-art techniques through extensive experiments involving both synthetic and real data.
\end{itemize}

\subsection{Organization of the Paper}
In Section \ref{ProbDes}, we describe the contextual multi-armed bandit framework. Next, we explain a first mixture of experts based approach and its challenges in Section \ref{Algorithm}. In Section \ref{BT}, we explain the notion of hierarchical structures and implement our algorithm using these structures. We introduce an efficient quantization method in Section \ref{EQ}, and show that our algorithm is competitive against any mapping, including the best arm selection policy, from the context space to the bandit arms. Section \ref{simu} contains the experimental results over several synthetic and well known real life datasets followed by the concluding remarks in Section \ref{Conc}.  \par 

\section{Problem Description}\label{ProbDes}
We study the contextual bandit problem in an adversarial setting\footnote{All vectors are column vectors and denoted by boldface lower case letters. For a $K$-element vector $\vec{u}$, ${u}_i$ represents the $i^{\text{th}}$ element and $\lVert \vec{u} \rVert=\sqrt{\vec{u}^T\vec{u}}$ is the $l^2$-norm, where $\vec{u}^T$ is the transpose. Indicator function $\mathbf{1}_{\lbrace {\cdot} \rbrace}\in \{0,1\}$ outputs $1$ only if its argument condition holds. A function $f:\mathbb{R}^n \to \mathbb{R}$ is Lipschitz continuous over a region $W \subset \mathbb{R}^n$, if there exists a non-negative constant $c$ such that $\lvert f(\vec{x}_1)-f(\vec{x}_2) \rvert \leq c\lVert \vec{x}_1-\vec{x}_2 \rVert$ for all $\vec{x}_1 , \vec{x}_2 \in W$.}. Recall that the original multi-arm bandit problem is a sequential game. One of the available bandit arms $I_t \in \lbrace 1,...,M \rbrace$ is selected at each round $t$ and then a related loss $l_{t,I_t}$ is observed\footnote{We assume $l_{t,I_t} \in [0,1]$ for simplicity, however, it can be straightforwardly shown that our results hold for any bounded loss after shifting and scaling in magnitude.}. The objective is to minimize the accumulated loss $\sum_{t=1}^{T} l_{t,I_t}$ in a sequence of $T$ rounds. In the contextual extension, a context vector $\vec{s}_t$ from a context space $S$ is additionally provided at each round before selecting the arm. For example, $S$ is $[0,1]^2$ in Fig. \ref{fig:1}. Then the objective stays same but can be improved with the available context.

We consider this contextual bandit problem in adversarial setting  without making any statistical assumptions about the context vectors and the bandit arms \cite{Auerthenonstochastic}, and propose algorithms that are guaranteed to work in an individual sequence manner. Our algorithms are strictly sequential such that at each round $t$, they select an arm $I_t$ according to the information coming from the previous rounds including observed context vectors, selected arms and their losses, alongside the context vector we are currently observing, i.e.,
\begin{equation}
I_t=f_t(\vec{s}_t;\vec{s}_{t-1},I_{t-1},l_{t-1,I_{t-1}};...;\vec{s}_{1},I_{1},l_{1,I_{1}}). 
\end{equation}
In design of our algorithms, we aim at sequentially learning the optimal partitioning of the context space with the optimal assignment between the regions of the learned partition and the set of arms. For this purpose, we investigate a general framework of hierarchical structures to generate context space partitions and eventually learn the asymptotically optimal, time varying, context driven arm chooser $f_t$. We show that our approach, compared to the state-of-the-art techniques, yields computationally highly superior algorithms with real time data processing capabilities while achieving a faster convergence rate to the optimal conditions (in terms of the convergence of the regret upper bounds to $0$). The superiority of the proposed algorithms is due to that the set of all possible context space partitions considered here can theoretically achieve arbitrarily high degree of granularity (can be of arbitrarily high capacity) whereas the true complexity of the optimal partition is limited (cf. Section \ref{BT}) in reality. Based on this observation, our approach additionally allows the regret analysis to incorporate an upper bound on the complexity of the optimal partition, which in turn significantly improves the convergence of the presented algorithms in almost all practical scenarios. This gain is essentially from ${O}(\sqrt{N})$ to ${O}(\sqrt{\ln N})$ ($N$ is measuring the granularity, cf. Section \ref{BT}). If the complexity of the optimal partition cannot be upper bounded, which would be a purely theoretical consideration as the true complexity is almost always limited and finite in real scenarios, our regret analysis then produces similar rates of convergence in that very worst theoretical scenario. Nevertheless, in any case, the proposed algorithms are computationally highly efficient and superior, and asymptotically optimal in the adversarial setting including the very worst scenario regardless of the stationary or non-stationary or perhaps chaotic source statistics.

To this end, we consider a large class $\mathcal{G}$ of deterministic mappings, i.e., $\forall g\in\mathcal{G}$, $g:S \to {\lbrace 1,...,M \rbrace}$. Each such mapping is composed of a fixed partition of the context space and an arm is assigned to each partition region. Depending on the partition region that a context $s_t$ falls in, $g$ chooses the assigned arm $g(s_t)$. An example is shown in Fig. \ref{fig:1a} in the case of $2$ dimensional context space $S=[0,1]^2$ with $2$ bandit arms, where $g([0.5, 0.5]^T)=1$. Note that for a given $g\in\mathcal{G}$, all of the other deterministic mappings resulting from all possible arm assignments to the regions of the partition of $g$ are also included in $\mathcal{G}$. Since we work in the adversarial setting and therefore refrain from making any statistical assumptions about the context vectors and the loss of the bandit arms \cite{Auerthenonstochastic}, we next define our performance w.r.t. the optimum (minimum loss) mapping in the ``competition" class $\mathcal{G}$ based on the following regret:
\begin{equation}
\mathcal{R}(T,\mathcal{G}) \triangleq \max_{g \in \mathcal{G}} \mathbb{E} \left\lbrack \sum_{t=1}^{T} l_{t,I_t} - \sum_{t=1}^{T} l_{t,g(\vec{s}_t)} \right\rbrack,
\end{equation}
where the expectation is w.r.t. the internal randomization in our algorithms (the internal randomization here is not related to data statistics). Our goal is to upper bound the regret by a term that depends sublinearly in $T$, and hence asymptotically achieve -at least- the performance of the best $g$ in $\mathcal{G}$ (in the averaged regret per round sense). Achieving this goal is equivalent to achieving the performance of the chooser of the optimal context space partition with the optimal assignment to the arms. Here, optimality of the context space partition should be understood w.r.t. the class $\mathcal{G}$ which is certainly not restrictive, since it can be arbitrarily improved by generalizing (detailing) $\mathcal{G}$ to a desired degree, cf. Section \ref{Algorithm}. \par  

We next construct the class $\mathcal{G}$ and provide a mixture-of-experts based first solution to the introduced problem.

\section{A Contextual Bandit Algorithm Based on Mixture of Experts}\label{Algorithm}
\begin{figure}[!t]
	\centering
	\includegraphics[width=0.4\textwidth]{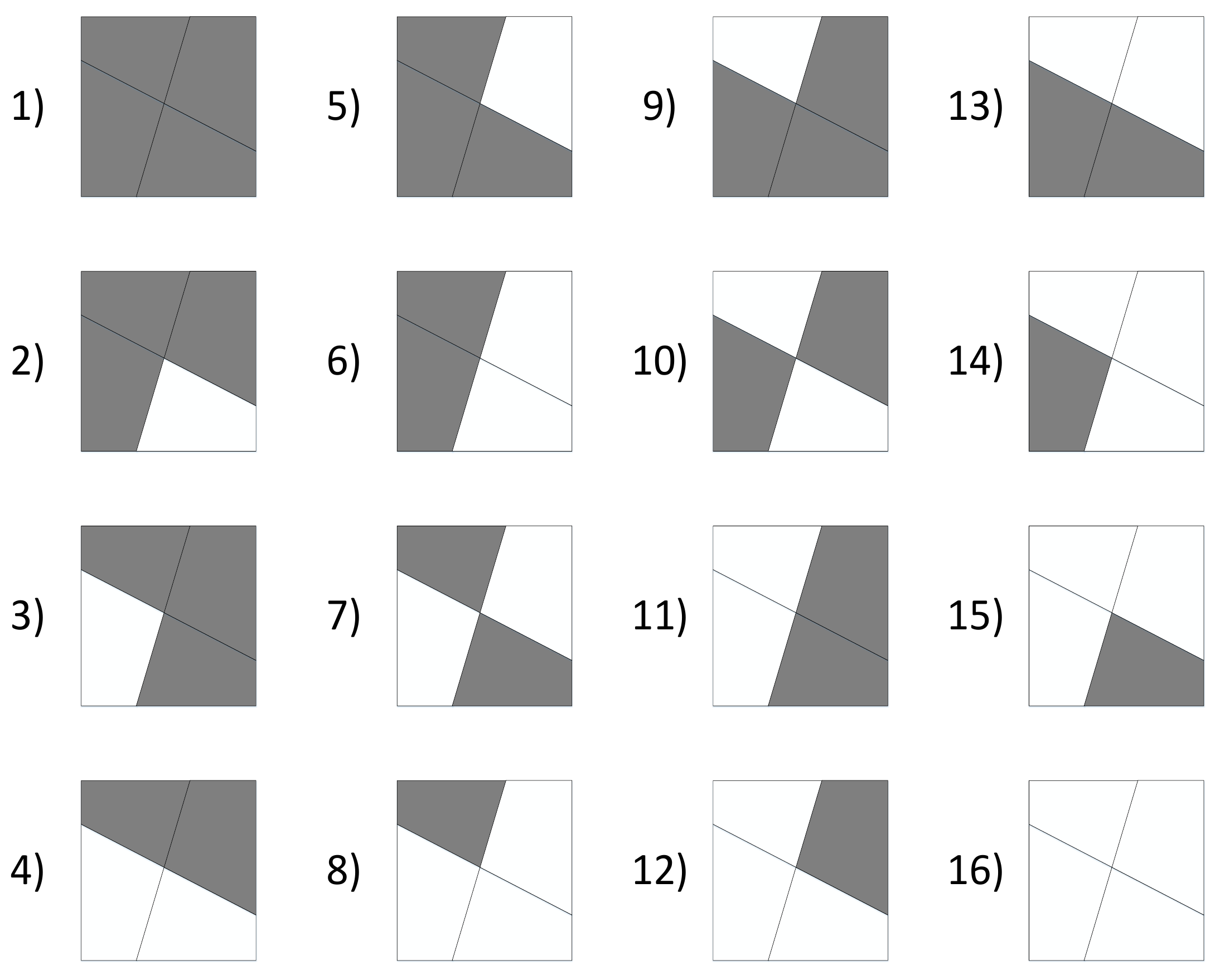}\\
	\caption{All possible mappings in a $2$-armed bandit problem with a predetermined quantization of the context space $S=\lbrack 0,1 \rbrack ^2$ into $4$ regions. In each mapping above, the dark and bright regions are mapped to the arms $1$ and $2$, respectively. }
	\label{allexp}
\end{figure}
The ultimate goal in the contextual bandit problem is ideally to achieve the performance of the best mapping {\em in the set $\mathcal{U}$\footnote{This set $\mathcal{U}$ consists of {\em all possible arbitrary} context space partitions (not confined to $\mathcal{G}$) with all possible assignments of partition regions to the arms.} of all arbitrary mappings} from the context space to the bandit arms. Since this set of all arbitrary mappings is too powerful to compete against in design of an algorithm, as the first step, we uniformly quantize the context space $S$ into $N$ disjoint regions $r_1,r_2,...,r_N$, i.e., $\cup_{i=1}^N r_i = S$ and $r_i \cap r_j = \varnothing$ for $\forall i \neq j$. We use uniform quantization for simplicity, however, one can incorporate any arbitrary type of quantization into our framework straightforwardly. In our framework, we consider all possible assignments between the set of disjoint regions and the set of bandit arms, and call each context mapping resulting from one of those assignments an $N$-level quantized mapping. Therefore, each $N$-level quantized mapping is essentially a function from $\cup_{i=1}^N r_i = S$ to $\lbrace 1,...,M \rbrace$: a context $s\in r^* \subset S$ is mapped to the bandit arm that the region $r^*$ is assigned to. Two examples of such quantized mappings of different levels for the case of $2$-armed bandit with the context space $[0,1]^2$ are shown in Fig. \ref{fig:1b} and Fig. \ref{fig:1c}. Given a quantized context space $S=\cup_{i=1}^N r_i$, we define the class $\mathcal{G}^N$ of $N$-level quantized mappings as the ``competition class" with $N$ quantization levels consisting of all arbitrary assignments between the bandit arms and the given $N$ regions $\{r_i\}_{i=1}^N$.

{\bf Remark: }We seek to achieve the performance of the best quantized mapping in $\mathcal{G}^N$, which can get arbitrarily close (and $N$ can be freely chosen in our framework) to the performance of the {\em best arbitrary mapping in $\mathcal{U}$, i.e., the best arm selection policy,} as $N$ increases. For example, suppose that the mapping shown in Fig. \ref{fig:1a} is the {\em best arbitrary mapping}. In this case, the mappings in Fig. \ref{fig:1b} and Fig. \ref{fig:1c} of improving optimalities will be the best mappings in $\mathcal{G}^{16}$ and $\mathcal{G}^{64}$, respectively.\par 

Based on $M^N$ different mappings in $\mathcal{G}^N$, we consider an expert chooser that is one-to-one-corresponding to each of those mappings such that $g_{j}(s)$ is the arm chosen by expert $E_j$ for the context $s$, i.e., $E_j \leftrightarrow g_j, 1 \leq j \leq M^N$. An example of all $16$ mappings followed by the experts for the case of $M=2$ and $N=4$ is shown in Fig. \ref{allexp}, where, unlike Fig. \ref{fig:1}, we choose a nonuniform quantization to demonstrate the generality in our approach. One of these experts in Fig. \ref{allexp} is $\mathcal{G}^4$-optimal for the underlying sequence of losses, however, naturally, we do not know which. Hence, instead of committing to a single expert, we next use a mixture of experts approach to learn the best one during rounds.\par
In order to achieve the performance of the best expert, we assign each expert $E_j$ a weight $\alpha_{t,j}$ (showing our trust on the expert $E_j$ at round $t$) and use exponentiated weights to adaptively combine them. After observing context $\vec{s}_t$ at each round $t$, we randomly select one of the experts using the probability simplex $\boldsymbol{\beta}_t=(\beta_{t,1},...,\beta_{t,M^N})$, where $\beta_{t,j}=\alpha_{t,j}/{\sum_{k=1}^{M^n} \alpha_{t,k}}$ is the normalized weight. Importantly, the probability of selecting each arm then follows the probability simplex $\vec{p}_t=(p_{t,1},...,p_{t,M})$, where
\begin{equation} \label{mixture}
p_{t,i}={\sum_{j=1}^{M^N} \beta_{t,j} {\mathbf{1}}_{\lbrace g_j(\vec{s}_t)=i \rbrace}}.
\end{equation}
We initially set the weights $\alpha_{1,i}$ according to the complexity of the mappings of experts from $\mathcal{G}^N$, and use exponentiated losses to update during rounds: at each round $t \geq 2$, we have

\begin{equation} \label{expweight}
\alpha_{t,i}=\alpha_{1,i} e^{-\eta\sum_{\tau=1}^{t-1} \tilde{l}_{\tau,g_i(\vec{s}_\tau)}},
\end{equation}
where $\eta \in \mathbb{R}^{+}$ is the (constant) learning rate and $\tilde{l}_{\tau,g_i(\vec{s}_\tau)}$ is the unbiased estimator of ${l}_{\tau,g_i(\vec{s}_\tau)}$. Since we do not observe the loss $l_{t,m}$ of the unchosen arms, we use the unbiased estimator
\begin{equation} \label{unbiased}
\tilde{l}_{t,m}=\begin{cases} 
\frac{{l}_{t,m}}{p_{t,m}} & m = I_t \\
0 & m \neq I_t 
\end{cases},
\end{equation}
where $\mathbb{E}[\tilde{l}_{t,m}]=l_{t,m}$. Using this bandit arm selection probability assignment defined through ~\eqref{mixture}, ~\eqref{expweight} and ~\eqref{unbiased}, we have the following regret result.
\begin{theorem}\label{thm:general}
	Consider an $M$-armed contextual bandit problem. If the context space is quantized into $N$ disjoint regions, and experts $E_j$'s are following the $M^N$ possible mappings in $\mathcal{G}^N$ as described in Section \ref{Algorithm}, then $\mathcal{R}(T,E_j)$ satisfies
	\begin{equation} \label{exp4reg}
	\mathcal{R}(T,E_j) \leq \frac{\ln{(1/{\beta_{1,j}})}}{\eta}+\frac{MT\eta}{2}
	\end{equation}	
	based on the probability assignments defined through ~\eqref{mixture}, ~\eqref{expweight} and ~\eqref{unbiased}, where $T$ is the number of rounds, $\eta \in \mathbb{R}^{+}$ is the learning rate parameter in \eqref{expweight} and $\beta_{1,j}$ is the normalized initial weight of the $j^{\text{th}}$ expert $E_j$.
\end{theorem}
{\em Proof of Theorem 1} follows similar lines to the proof of Theorem 4.2 in \cite{Cesa} with certain variations due to our arbitrary initial weighting as opposed to uniform initial weights of the experts in \cite{Cesa}. The proof of our Theorem 1 is provided in Appendix \ref{app:1}.

\begin{figure}[t]
	\centering
	\includegraphics[width=0.45\textwidth]{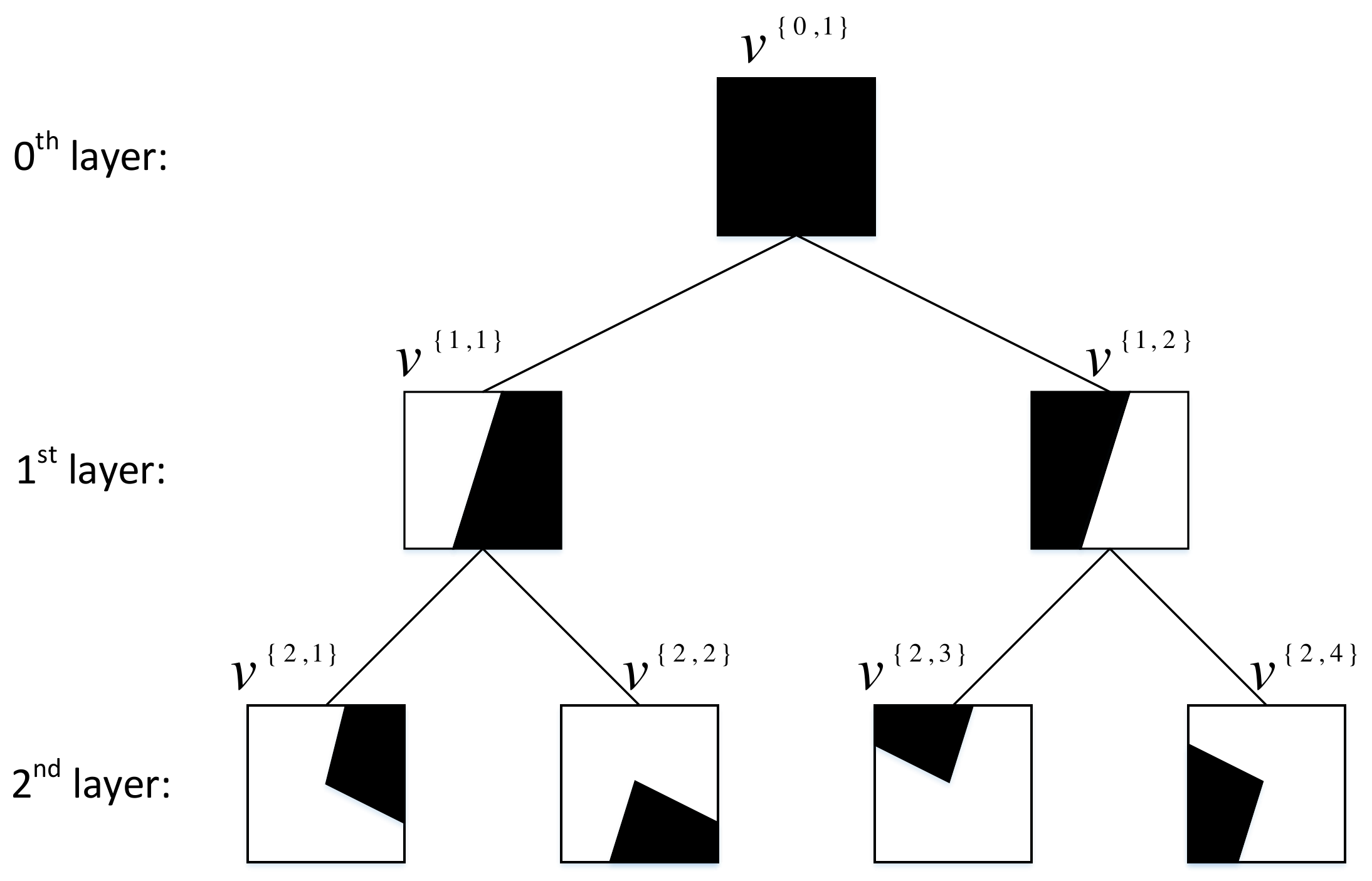}\\
	\caption{A binary tree of depth $D=2$ over the context space $\lbrack 0,1 \rbrack^2$. The regions corresponding to each node are filled with black color.}
	\label{binarytree}
\end{figure}
\begin{figure*}[t]
	\centering
	\includegraphics[width=0.70\textwidth]{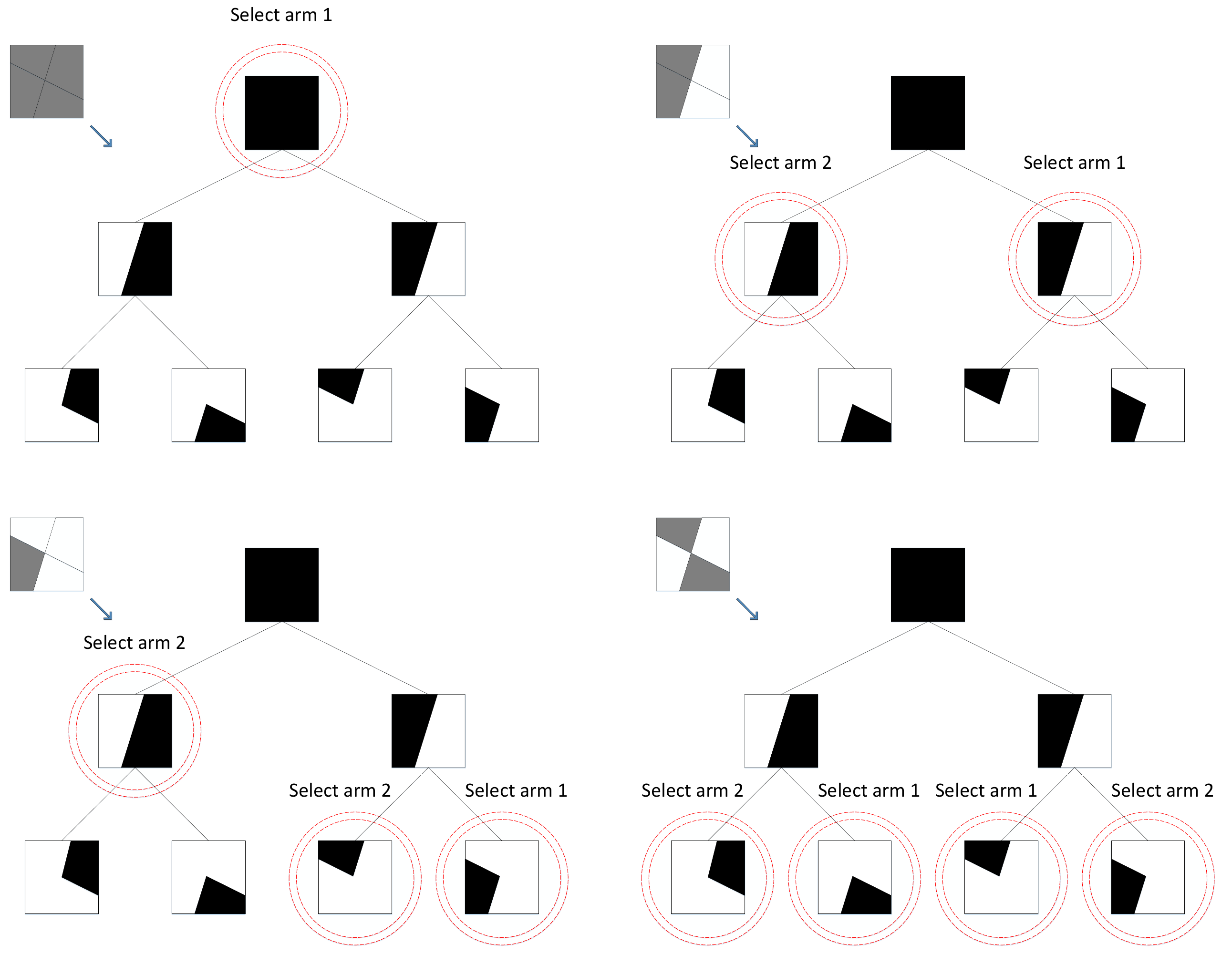}\\
	\caption{Representation of $4$ sample mappings in Fig. \ref{allexp} over the binary tree in Fig. \ref{binarytree}.}
	\label{exontree}
\end{figure*}

We observe that the regret bound is logarithmically dependent on the reciprocal of the prior weight of the optimal partitioning in the competition class (i.e., its complexity cost). Hence, by using equal prior weights on the $M^N$ experts, our regret bound will be in the order\footnote{For ease of exposition and simplicity in our order notation here, we drop the variables, on which the dependency of order is similar or same or negligible across the compared algorithms.} of $O(\sqrt{NT})$ (after optimizing the learning rate).
We point out that this result is similar to the \textit{EXP4} algorithm \cite{Cesa}, which achieves a regret upper bound of $O(\sqrt{NT})$ with optimum selection of the learning rate. Furthermore, \textit{S-EXP3} algorithm \cite{Cesa} achieves a regret upper bound of the same order $O(\sqrt{NT})$ using an independent \textit{EXP3} algorithm over each quantized region of the context space.  This square root dependency of the regret bound on the quantization level is prohibitive and working against our motivation of approximating the performance of {\em the best arbitrary mapping} by freely increasing the number of quantization levels. Instead, we would like our regret bound to be dependent on the actual number $R$ of disjoint regions that is needed and sufficient to model the actual complexity of the best arbitrary mapping whatever the quantization level $N$ is. Hence, we want to achieve the order $O(\sqrt{RT})$. 
Moreover, working with these $M^N$ parameters $\alpha_{t,1},...,\alpha_{t,M^N}$ has quite high space and computational complexities of $O(M^N)$. 

To this end, we introduce hierarchical structures to generate context space partitions and exploit the level of complexity that is sufficient to model the best mapping over the introduced hierarchy. Thus, we achieve a regret upper bound with square-root dependency on the actual number of regions $R$ in a computationally highly superior manner with significantly low space complexity.

\section{Hierarchical Structures}\label{BT}
We use hierarchical structures to implement our contextual bandit algorithm {\em efficiently} in terms of both the regret upper bound convergence to $0$ in average loss per round sense as well as computational and space complexities. Suppose that we have $H$ nodes in a hierarchical structure labeled $v_i$, $i \in \lbrace 1,2,...,H \rbrace$. We assign each node $v_i$ a region $r_i$ from the context space and there is hierarchical connection from each parent node to its child nodes. Let $\Phi_i$ be the set of child node groups of the node $v_i$, where each group $\phi \in \Phi_i$ consists of child nodes such that the union of their corresponding regions gives the region associated with the parent node $v_i$. \par 
For instance, consider the binary tree of depth $2$ in Fig. \ref{binarytree}, which quantizes the $2$-dimensional context space $S=[0,1]^2$. Each node of such binary tree corresponds to a region of the context space, as shown in the figure. The region corresponding to each node is the union of the regions of its child nodes. Hence, for each node $v_i$ in this tree (except for the leaf nodes), the set $\Phi_i$ is of size $1$, which consists of only one group of cardinality $2$ (which is the parent node's child pair). For the leaf nodes, $\Phi_i$ is the empty set and, hence, has a size of $0$.\par 
Next, we use this hierarchical structure to compactly represent our experts and combine them in an efficient manner.

\subsection{A Weighted Mixture of Experts Algorithm Using Hierarchical Structures}\label{alg}

In the following, we explain the details of our efficient implementation of the mixture of experts algorithm (described in Section \ref{Algorithm}) by using hierarchical structures and present several examples. In addition to achieving computational scalability in our implementation, another goal of our work is to incorporate the model complexity of the best expert to improve the upper bound on the regret. \par 


Here, each expert is composed of a partition of the context space and an arm assigned to each partition region. The partition corresponding to each expert can be represented using several nodes of the hierarchical structure. Hence, each expert can be represented using several nodes (showing the partition) and an arm corresponding to each one of them (showing the arm assignments). As an example, consider a $2$-armed bandit problem. Suppose that we use a binary tree of depth $2$ to quantize the context space into $4$ regions. In this case, we define $2^4=16$ experts as in Fig. \ref{allexp}. We represent $4$ samples among these $16$ experts on our binary tree in Fig. \ref{exontree}. In this figure, the nodes representing the partition corresponding to the experts are marked using the circles and the arm selected by the expert at each one of these nodes is declared over the node. We seek to adaptively combine all of the experts to achieve the performance of the best one as explained in Section \ref{Algorithm}. \par 
In order to implement our mixture of experts, over each node $v_i$, we define $M$ parameters $\alpha_{t,m,i}$ for $m=1$ to $M$ as the weight of $m^{\text{th}}$ arm in the node $v_i$. This weight shows our trust on the $m^{\text{th}}$ arm when the context vector falls into the region corresponding to the node $v_i$. We set $\alpha_{1,m,i}=1$ for all $m$'s and $v_i$'s, and for $t \geq 2$,
%
\begin{equation} \label{eq:alphadef}
\alpha_{t,m,i}=\exp{\left( -\eta \sum_{\tau=1}^{t-1} \frac{l_{I_\tau}}{p_{\tau,m}} \mathbf{1}_{\left\lbrace I_\tau=m\right\rbrace} \mathbf{1}_{\left\lbrace \vec{s}_\tau \in r_i \right\rbrace} \right)}.
\end{equation}
We can easily update these weights as follows. At each round $t$, after we receive $\vec{s}_t$, calculate $\vec{p}_t$, select $I_t^{\text{th}}$ arm and observe the loss $l_{t,I_t}$, we calculate
\begin{equation} \label{eq:alphaupdate}
\alpha_{t+1,m,i}=\alpha_{t,m,i} \exp{\left( -\eta \frac{l_{I_t}}{p_{t,m}} \mathbf{1}_{\left\lbrace I_t=m\right\rbrace} \mathbf{1}_{\left\lbrace \vec{s}_t \in r_i \right\rbrace}\right)}.
\end{equation}

We point out that the weight of each expert $\alpha_{t,k}$ in \eqref{expweight} can be written as a multiplication of its initial weight and our weight parameters (i.e. $\alpha_{t,m,i}$'s) on the tree nodes corresponding to the mapping followed by the expert. 
To this end, in order to obtain the expert weights (cf. Theorem 2), we define another variable $w_{t,i}$ over each node $v_i$ such that
\begin{equation}\label{eq:wrecursion}
w_{t,i}= \frac{1}{(| \Phi_i |+1)M}\sum_{m=1}^{M}\alpha_{t,m,i}+
\frac{1}{| \Phi_i |+1}\sum_{\phi \in \Phi_i}^{}
\left(
\prod_{j \in \phi} w_{t,j}
\right).
\end{equation}
Hence, if $\Phi_i$ is the empty set (i.e. $| \Phi_i |=0$), then the equation simply becomes
\begin{equation}
w_{t,i}= \frac{1}{M}\sum_{m=1}^{M}\alpha_{t,m,i}.
\end{equation}
The following proposition shows that using this recursion to calculate $w_{t,i}$ variables, the weight of the root node $w_{t,1}$ becomes equal to the sum of the expert weights, i.e., $\sum_{k} \alpha_{t,k}$ (as defined in \eqref{expweight}).
\begin{proposition}\label{th:2}
	Using the recursive formula in \eqref{eq:wrecursion}, at each node $v_i$, we have
	\begin{equation}
	w_{t,i}=\sum_{k \in \Gamma_i}^{} \alpha_{t,k},
	\end{equation}
	where $\Gamma_i$ is the set of all experts defined over node $v_i$.
\end{proposition}
{\em Proof of Proposition 1} is provided in Appendix \ref{app:2}. \par
Now, in order to calculate the probability simplex in \eqref{mixture}, we define $M$ other variables to calculate $\sum_{k} \alpha_{t,k} \mathbf{1}_{\lbrace g_k(\vec{s}_t)=i \rbrace}$ for $i=1,...,M$. To this end, after we observe $\vec{s}_t$, we set
\begin{equation}
\gamma_{t,m,i}=\frac{1}{M}\alpha_{t,m,i},
\end{equation}
at the nodes $v_i$ containing $\vec{s}_t$, where $| \Phi_i |=0$ (i.e., leaf nodes). Then, we go up on the hierarchy using a recursive formula similar to the way we calculate $w_{t,i}$ variables in \eqref{eq:wrecursion} as
\begin{align} \label{eq:gammarecursion}
\gamma_{t,m,i}=& \frac{1}{(| \Phi_i |+1)M}\alpha_{t,m,i}\nonumber\\
+&\frac{1}{| \Phi_i |+1}\sum_{\phi \in \Phi_i}^{}
\left(
\prod_{j \in \phi} w_{t,j}
\left(
\frac{\gamma_{t,m,j}}{w_{t,j}}
\right)^{\mathbf{1}_{\lbrace \vec{s}_t \in r_j \rbrace}}
\right).
\end{align}
Using this recursion, we calculate $\gamma_{t,m,1}$ for $m=1,...,M$. The following proposition shows that using this recursion, $\gamma_{t,m,1}$ is the weighted sum of all experts, which select the $m^{\text{th}}$ arm when they observe $\vec{s}_t$. Hence, we can build the probability simplex in \eqref{mixture} as
\begin{equation}
p_{t,m}=\gamma_{t,m,1}/w_{t,1}, \forall m \in \lbrace 1,...,M \rbrace.
\end{equation}

\begin{proposition}
	Using the recursive formula in \eqref{eq:gammarecursion}, at each node $v_i$, for all $m \in \lbrace 1,...,M \rbrace$, we have
	\begin{equation} \label{eq:th3}
	\gamma_{t,m,i}=\sum_{k \in \Gamma_i}^{} \alpha_{t,k} \mathbf{1}_{\lbrace g_k(\vec{s}_{t})=m \rbrace},
	\end{equation}
	where $\Gamma_i$ is the set of all experts defined over node $v_i$.
\end{proposition}
{\em Proof of Proposition 2} is provided in Appendix \ref{app:3}. \par 

With the proposed implementation of the algorithm, at each round $t$, after observing $\vec{s}_t$, we first calculate $\gamma_{t,m,1}$ for $m=1,...,M$ and then divide by $w_{t,1}$ to form the probability simplex $\vec{p}_t=(p_{t,1},...,p_{t,m})$, using which we select an arm $I_t$. After we select our arm and suffer the loss according to the selected arm, we first update $\alpha_{t,I_t,i}$ parameters at the nodes containing $\vec{s}_t$. Then, we update $w_{t,i}$ variables at these affected nodes and go to the next round.
The pseudo code of the explained procedure is provided in Algorithm \ref{tab:table1}. \par 
\begin{algorithm}[t]
	\caption{Hierarchical Structure based Bandits (\textit{HSB})}
	\label{tab:table1}
	\begin{algorithmic}[1]
		\algsetup{linenosize=\tiny}
		\STATE \textbf{Parameter:}
		\STATE Set constant $\eta \in \mathbb{R}^{+}$
		\STATE \textbf{Initialization:}
		\STATE Initialize the structure including nodes $v_i$, the regions $r_i$ and the hierarchical relations $\Phi_i$.
		\STATE Initialize $\alpha_{1,m,i}=1$ for all $m,i$.
		\STATE Initialize $w_{1,i}$ for all $i$ using \eqref{eq:wrecursion}
		\STATE \textbf{Algorithm:}
		\FOR {$t=1$ \TO $T$ }
		\STATE Observe $\vec{s}_t$
		\FOR {$m=1$ \TO $M$}
		\STATE Calculate $\gamma_{t,m,i}$ according to \eqref{eq:gammarecursion}
		\ENDFOR
		\FOR {$m=1$ \TO $M$}
		\STATE $p_{t,m} = \gamma_{t,m,1} / w_{t,1}$
		\ENDFOR
		\STATE Select a random arm $I_t$ according to the probability simplex $\vec{p}_{t}=(p_{t,1},...,p_{t,M})$
		\STATE Set $\alpha_{t+1,m,i}=\alpha_{t,m,i}$ for all $m,i$
		\STATE Set $w_{t+1,i}=w_{t,i}$ for all $i$
		\FOR {the nodes $v_i$, where $\vec{s}_t \in r_i$}
		\STATE Calculate $\alpha_{t+1,I_t,i}$ according to \eqref{eq:alphaupdate}
		\ENDFOR
		\FOR {the nodes $v_i$, where $\vec{s}_t \in r_i$}
		\STATE Calculate $w_{t+1,i}$ using \eqref{eq:wrecursion}
		\ENDFOR
		\ENDFOR
	\end{algorithmic} 	
\end{algorithm} 
Next, we show the regret bound of our hierarchical structure algorithm.
\begin{theorem}
	Algorithm \ref{tab:table1} achieves the regret bound
	\begin{equation}
	\mathcal{R}(T,\mathcal{G}^N)\leq 
	\frac{\Psi(A_R+1) \ln((H_S+1)M)}{\eta}
	+\frac{MT\eta}{2},
	\end{equation}
	where $\Psi$ is an upper bound on the cardinality of the child node groups $\phi$, i.e., $\Psi \geq |\phi|$ for all $\phi$, $H_S$ is an upper bound on the cardinality of $\Phi_i$, i.e., $H_S \geq |\Phi_i|$ for all $i$, and $A_R$ is an upper bound on the minimum number of splittings needed in the hierarchical structure to model the optimal partition with $R$ disjoint regions.
\end{theorem}
\begin{proof} [Proof of Theorem 2]
	If the optimal expert is defined over the root node, i.e., $A_R=0$, its prior weight in the mixture is
	\begin{equation}
	\beta_{1,j}=\frac{1}{(|\Phi_i|+1)M} \geq \frac{1}{(H_S+1)M}.
	\end{equation}
	With each split in the hierarchical structure (i.e., with each move down the hierarchy), the prior weights of the experts are divided by a factor which is at most $(H_S+1)^{\Psi}M^{\Psi-1}$. Thus, in case we need $A_R$ splittings to model the partition corresponding to the optimal expert, its prior weight is
	\begin{align}
	\beta_{1,j} \geq (H_S+1)^{-A_R\Psi-1}M^{A_R-A_R\Psi-1}.
	\end{align}
	Since $A_R \geq 1$ and $\Psi \geq 1$, we have
	\begin{align}
	\beta_{1,j} \geq (H_S+1)^{-\Psi (A_R+1)}M^{-\Psi (A_R+1)}.
	\end{align}
	Hence,
	\begin{equation} \label{eq:th4proof}
	\ln(1/\beta_{1,j}) \leq 
	-\Psi (A_R+1) \ln((H_S+1)M).
	\end{equation}
	Putting \eqref{eq:th4proof} into \eqref{exp4reg} concludes the proof.
\end{proof}
\begin{corollary}
	By setting
	\begin{equation}
	\eta=\sqrt{\frac{2\Psi(A_R+1) \ln((H_S+1)M)}{MT}},
	\end{equation}
	we get the regret bound of
	\begin{equation} \label{eq:regretbound}
	\mathcal{R}(T,\mathcal{G}^N) \leq \sqrt{0.5\Psi MT (A_R+1) \ln{((H_S+1)M)}}.
	\end{equation}
\end{corollary}
We next present several examples of hierarchical structures which can be employed by our algorithm with the introduced mathematical guarantees. Each structure has its own way of encoding the best arm selection policy, i.e., optimal arbitrary mapping. Hence, the proper selection of the hierarchical structure according to the target application leads to a smaller $A_R$ and a better performance, i.e., a regret upper bound vanishing faster in the average loss per round sense, together with the introduced weighting over the corresponding competition class $\mathcal{G}^N$, cf. Section \ref{simu} as well as the examples below.
\subsection{Example 1: Arbitrary Splitting}
If the hierarchical structure is an arbitrary splitting of $N$ leaf nodes into $2$ groups, then $\Psi=2$, $H_S=2^{N-1}-1$ and $A_R=M-1$. Hence, the regret is upper bounded as
\begin{align}
\mathcal{R}(T,\mathcal{G}^N)&\leq 
\frac{2M \ln(2^{N-1}M)}{\eta}
+\frac{MT\eta}{2} \nonumber\\
&\leq
\frac{2MN \ln(M)}{\eta}
+\frac{MT\eta}{2},
\end{align}
where the last inequality uses $2 \leq M$.

\subsection{Example 2: Binary Tree}
In binary trees we have $\Psi=2$ and $H_S=1$. For a binary tree with $N$ leaf nodes, we need at most $\log_2 N$ splitting to create each new region. Hence, $A_R=(R-1)\log_2 N$. Therefore,
\begin{align}
\mathcal{R}(T,\mathcal{G}^N)&\leq 
\frac{2((R-1)\log_2 N+1) \ln(2M)}{\eta}
+\frac{MT\eta}{2} \nonumber\\
&\leq
\frac{2R \log_2 N \ln(2M)}{\eta}
+\frac{MT\eta}{2}.
\end{align}

\subsection{Example 3: K-ary Tree}
If the hierarchical structure is a K-ary tree (for $K=2$ this becomes a binary tree) with $N$ leaf nodes and depth $D=\log_K N$, then $\Psi=K$, $H_S=1$ and $A_R=(R-1)\log_K N$. Therefore, we have
\begin{align}
\mathcal{R}(T,\mathcal{G}^N)&\leq
\frac{K(1+(R-1)\log_K N) \ln(2M)}{\eta}
+\frac{MT\eta}{2} \nonumber\\ 
&\leq \frac{KR\log_K N \ln(2M)}{\eta}
+\frac{MT\eta}{2}.
\end{align}

\subsection{Example 4: Lexicographical Splitting Graph}
In a lexicographal splitting graph with $N$ leaf nodes, we have $\Psi=2$, $H_S=N-1$ and $A_R=R-1$. Hence,
\begin{align}
\mathcal{R}(T,\mathcal{G}^N)\leq
\frac{2R \ln(NM)}{\eta}
+\frac{MT\eta}{2}.
\end{align}

\subsection{Example 5: K-group Lexicographical Splitting}
If the hierarchical structure is a splitting of $N$ sequentially ordered leaf nodes into $K$ groups (when $K=2$ this structure becomes the lexicographical splitting graph), then $\Psi=K$, $H_S= {{N-1}\choose{K-1}}$ and $A_R=\lceil \frac{R-1}{K-1} \rceil$. Therefore, the regret upper bound is
\begin{align}
\mathcal{R}(T,\mathcal{G}^N)&\leq
\frac{K(\lceil \frac{R-1}{K-1} \rceil+1) \ln((1+{{N-1}\choose{K-1}})M)}{\eta}
+\frac{MT\eta}{2} \nonumber\\ 
&\leq \frac{K(R+2K) \ln(NM)}{\eta}
+\frac{MT\eta}{2}.
\end{align}

\subsection{Example 6: Arbitrary Position Splitting}
In this case, for a $d$-dimensional context space, we have $\Psi=2$, $H_S=d$ and $A_R=(R-1)\log_2 N$. Therefore,
\begin{align}
\mathcal{R}(T,\mathcal{G}^N)&\leq 
\frac{2((R-1)\log_2 N+1) \ln((d+1)M)}{\eta}
+\frac{MT\eta}{2} \nonumber\\
&\leq
\frac{2R \log_2 N \ln((d+1)M)}{\eta}
+\frac{MT\eta}{2}.
\end{align}

We have successfully achieved a regret bound of $O(\sqrt{MTR\ln N \ln{M}})$ with proper selection of the learning rate. Note that typically, $N>>R$. Our regret bounds are only logarithmically dependent on $N$, hence, in soft-$O$ notation, we achieve the minimax optimal regret bound $\tilde{O}(\sqrt{TR})$.\par 
Next and finally, we address the goal of achieving the performance of the best arm selection policy, i.e., the performance of the optimal arbitrary mapping (in the ultimate set $\mathcal{U}$) from the context space to the bandit arms which is not necessarily in the competition class $\mathcal{G}^N$ but can be approximated arbitrarily well and almost perfectly, if desired, by the class by increasing $N$. The quantization process in our algorithm naturally produces an additive linear-in-time term in our regret against the truly optimal mapping in $\mathcal{U}$. In the following section, we assume  that the arm losses are Lipschitz continuous in the context vectors at each specific round. With this assumption, we show that using a uniform quantization of the context space, we can diminish the linear-in-time term in our regret against the optimal mapping in $\mathcal{U}$ by increasing the number of quantization levels $N$. Hence, we can achieve a performance as close as desired to the performance of the optimal mapping in $\mathcal{U}$.

\section{An Efficient Quantization Method to Asymptotically Achieve the Optimal Context Based Arm Selection}\label{EQ}
Suppose that the context space is the $n$-dimensional space $S=\lbrack 0,1 \rbrack^n$. Using a hierarchical structure with $N$ leaf nodes, our quantization scheme is as follows. We split the context space into $2^{\lfloor (\log_2 N)/n \rfloor+1}$ equal subspaces along the first $\log_2 N \;(\bmod\; n)$ dimensions (of the total $n$ dimensions), and $2^{\lfloor (\log_2 N)/n \rfloor}$ equal subspaces along the remaining dimensions. 
\begin{theorem}
Using aforementioned quantization method for our algorithm, if the arm loss functions are Lipschitz continuous with the Lipschitzness constant $c$, then the difference between the loss corresponding to the best mapping in $\mathcal{G}^N$ and the loss corresponding to the truly optimal mapping (in the ultimate set\footnote{This ultimate set can be non-rigorously considered as $\mathcal{G}^{\infty}$.} $\mathcal{U}$ of all possible arbitrary mappings from the context space to the set of bandit arms) is upper bounded by
\begin{equation}
\frac{2c\sqrt{n}} {\sqrt[n]{N}}.\label{lipup}
\end{equation}
\end{theorem}
\begin{proof}[Proof of Theorem 3]
	Using this quantization method, the subspaces in the finest partition of the context space are $n$-dimensional cubes with the longest diagonal length equal to
	\begin{equation}
	\sqrt{{\frac{n-(\log_2 N \;(\bmod\; n))}{({2^{\lfloor \frac{\log_2 N}{n} \rfloor}})^2}}+{\frac{\log_2 N \;(\bmod\; n)}{(2^{{\lfloor \frac{\log_2 N}{n}+1 \rfloor}})^2}}}.
	\end{equation}
	Since $\log_2 N \;(\bmod\; n) \geq 0$, this upper bound is at most equal to
	\begin{equation}
	\sqrt{\frac{n}{2^{2{\lfloor \frac{\log_2 N}{n} \rfloor}}}} \leq \frac{2\sqrt{n}}{2^{\frac{\log_2 N}{n}}}=\frac{2\sqrt{n}} {\sqrt[n]{N}}.
	\end{equation}
	Since the loss functions are Lipschitz continuous, the difference between the loss corresponding to the truly optimal mapping in $\mathcal{U}$ and the best mapping in $\mathcal{G}^N$ cannot exceed the Lipschitzness constant times the quantized cubes diagonal length, which concludes the proof.
\end{proof}
Note that the Lipschitzness assumption does not intervene with the adversarial setting. The loss functions can be quite different in different rounds and as long as they are Lipschitz continuous at each specific round, the assumption holds and our algorithm is competitive against the ultimate set of all possible arbitrary mappings $\mathcal{U}$. In this case, combining \eqref{lipup} with the regret bound in \eqref{eq:regretbound} directly concludes the following theorem.
\begin{theorem}
	Consider a contextual $M$-armed bandit problem with the context space $S=[0,1]^n$, where the loss functions of the arms are Lipschitz continuous with the constant $c$ at all rounds. If we use a hierarchical structure with $N$ leaf nodes following the quantization scheme described in Section \ref{EQ}, the regret of Algorithm \ref{tab:table1} against the truly optimal strategy in a $T$ round trial is upper bounded as follows
	\begin{equation}
	R{(T,\mathcal{U})} \leq \sqrt{\frac{\Psi MT (A_R+1) \ln{((H_S+1)M)}}{2}}+\frac{2Tc\sqrt{n}} {\sqrt[n]{N}}.\label{bound}
	\end{equation}
\end{theorem}
We emphasize that we can make the linear-in-time term of the upper bound in \eqref{bound} as small as desired by growing the hierarchical structure and increasing the number of leaf nodes $N$, which is equal to the number of quantization levels.

\section{Experiments}\label{simu}

In this section, we demonstrate the performance of our algorithm in different scenarios involving both real and synthetic data. We demonstrate the performance of our main algorithm \textit{HSB} with various hierarchical structures including binary tree (\textit{HSB-BT}), lexicograph (\textit{HSB-LG}) and arbitrary position splitting (\textit{HSB-APS}) \cite{ContextWeighting}. We compare the performance of our algorithms against the state-of-the-art adversarial bandit algorithms \textit{EXP3} and \textit{S-EXP3} \cite{Cesa}. In all of the experiments, the parameters of \textit{EXP3} and \textit{S-EXP3} algorithms are set to their optimal values according to their publication \cite{Cesa}. 
\subsection{Stationary Environment}\label{stationary}
\begin{figure}[t]
	\centering
	\includegraphics[width=0.40\textwidth]{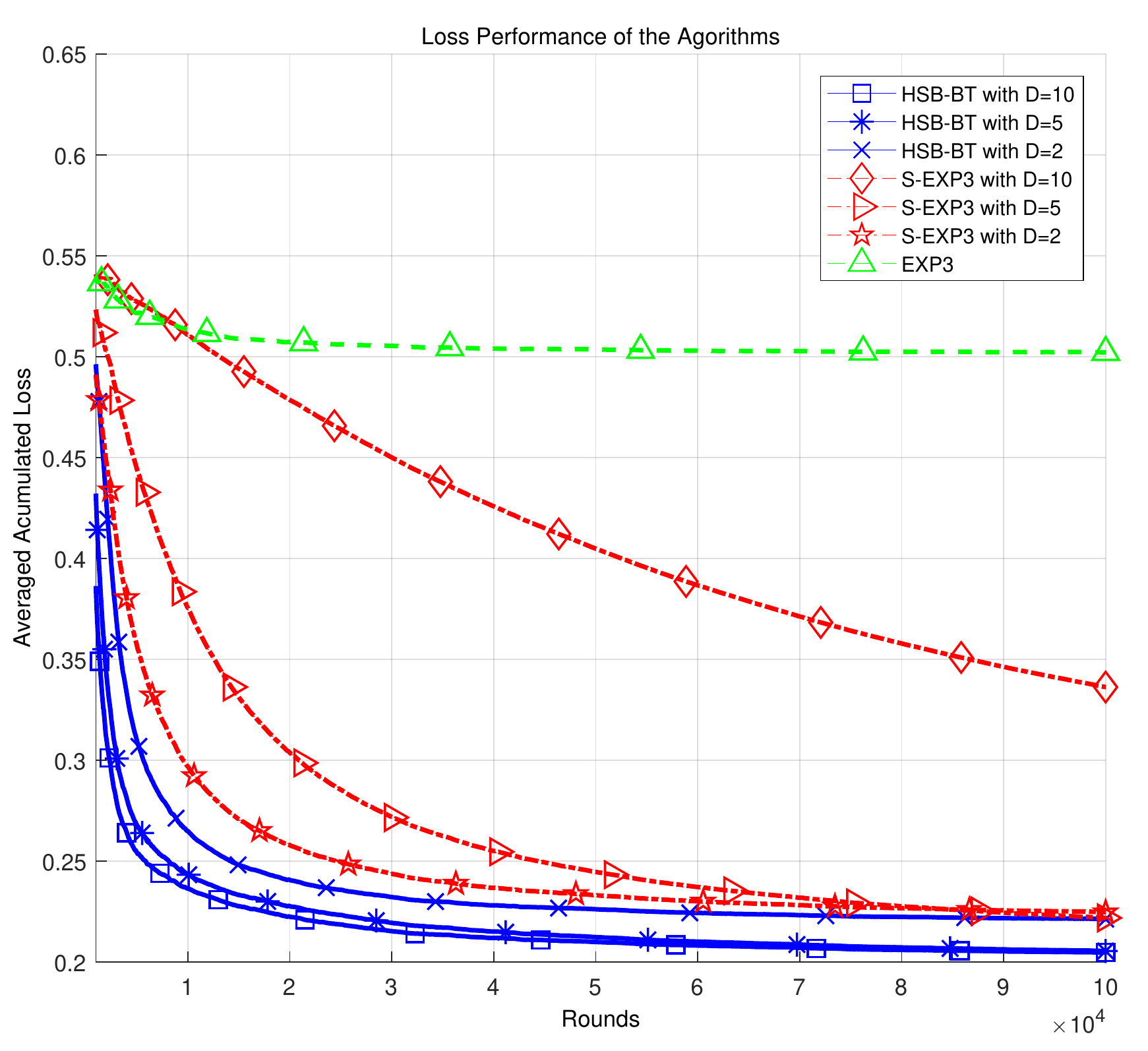}\\
	\caption{The averaged accumulated loss of \textit{HSB-BT}, \textit{S-EXP3}, and \textit{EXP3} on the datasets defined using (\ref{model}).}
	\label{figure6}
\end{figure} 
We first construct a game with $3$-armed bandit, where the context space is the $1$-dimensional space $S=[0,1]$. Each arm $i$ generates its loss according to a Bernoulli distribution with parameter $p_i$, i.e., the loss is equal to $1$ with probability equal to $p_i$. These parameters, i.e., $p_1,p_2,p_3$, depend on the context variable ${s}_t$ as 
\begin{align}\nonumber 
&p_1(s_t)=0.5+0.5\sin(2\pi s_t),\\\nonumber
&p_2(s_t)=\sin(\pi s_t),\\
&p_3(s_t)=s_t.
\label{model}
\end{align}
Here, the optimal strategy is defined as follows
\begin{equation}
g(s_t)=\label{rparameter}
\begin{cases} 
3, & s_t<0.5 \\
1, & 0.5\leq s_t < 0.9182 \\
2, & 0.9182 \leq s_t.
\end{cases}
\end{equation}

In this experiment, we generate the context variable $s_t$ randomly with uniform distribution over the context space, i.e., $[0,1]$, and compare the averaged cumulated loss performance, i.e., $(\sum_{\tau=1}^{t} l_{\tau,I_\tau}) /t$, for our algorithm \textit{HSB-BT} with various depth parameters equal to $2$, $5$, and $10$, \textit{S-EXP3} \cite{Cesa} with the same depth parameters, and \textit{EXP3} \cite{Cesa}. \par 
To this end, we generate $10$ synthetic datasets of length $10^5$. To produce each dataset, first, $10^5$ context variables $s_t$ are drawn according to a uniform probability distribution over the interval $[0,1]$. Then, the arm losses corresponding to different rounds are drawn from the Bernoulli distributions, parameters of which are determined according to \eqref{model}. Each dataset is presented to the algorithms $10$ times and the results are averaged. This process is repeated for all $10$ datasets and the ensemble averages are plotted in Fig. \ref{figure6}. Two important results can be derived from the result of this experiment. First, our algorithm \textit{HSB-BT} outperforms both of the \textit{S-EXP3} and \textit{EXP3} algorithms. Second, while increasing the depth uniformly improves the performance of our algorithm, it can degrade the performance of \textit{S-EXP3} due to the overtraining. The superior performance of our algorithm in this experiment is because of its fast convergence to the optimal mapping. Here, \textit{EXP3} has a fast convergence but it converges to a suboptimal mapping because it does not use the context information. On the other hand, \textit{S-EXP3} converges to the optimal mapping, but needs a huge amount of data to get trained. Our algorithm uses an efficient adaptive combination of the experts with intelligent initial weights to obtain the advantages of both \textit{EXP3} and \textit{S-EXP3} algorithms, while mitigating their disadvantages.

\subsection{Nonstationary Environment}\label{suddenchange}
In this part, we illustrate the averaged cumulated loss performance of the algorithms in a nonstationary environment. To this end, we construct $10$ different datasets of length $10^5$ as in Section \ref{stationary}. However, here the arm losses follow a model as in \eqref{model} in the first quarter of the rounds, and the following model in the rest of the rounds: 
\begin{align}\nonumber
&p_1(s_t)=\sin(\pi s_t),\nonumber\\
&p_2(s_t)=s_t,\nonumber\\
&p_3(s_t)=0.5+0.5\sin(2\pi s_t).
\label{model2}
\end{align}
\begin{figure}[t]
	\centering
	\includegraphics[width=0.45\textwidth]{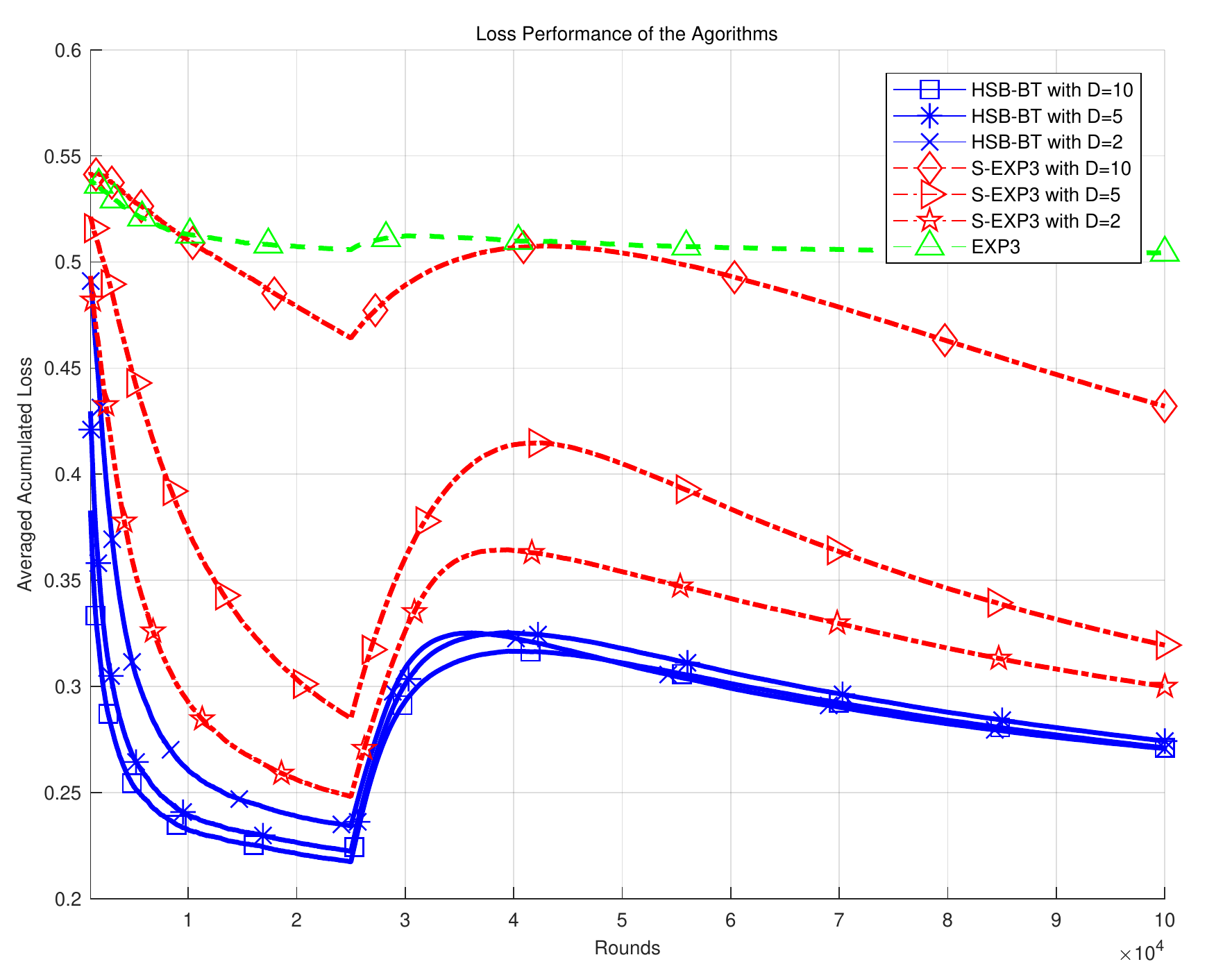}\\
	\caption{The averaged accumulated loss of \textit{HSB-BT}, \textit{S-EXP3}, and \textit{EXP3} on the datasets as described in Section \ref{suddenchange}, involving a rapid change in the behavior of the arms after $25\%$ of the rounds.}
	\label{figure8}
\end{figure}
Hence, we have an abrupt change in the model of the arms within the rounds. Each dataset is presented to the algorithms $10$ times and the results are averaged. This process is repeated for all $10$ datasets and the ensemble averages are plotted in Fig. \ref{figure8}. As shown in the figure, our algorithm \textit{HSB-BT} not only outperforms its competitor before the rapid change in the model of the bandit arms but also adopts better to this rapid change in comparison to the competitors. 
\subsection{Real Life Online Advertisement Dataset}
\begin{algorithm}[t]
	\caption{The offline evaluation method used to test the competitor algorithms over the Yahoo! Today Module dataset}
	\label{tab:table4}
	\begin{algorithmic}[1]
		\algsetup{linenosize=\tiny}
		\STATE \textbf{Input:} Bandit algorithm $\mathcal{A}$, logged data for $T$ rounds
		\STATE \textbf{Initialize:} $L=0$ and $R=0$
		\FOR {$t=1$ \TO $T$ }
		\STATE Get $s_t \in \lbrace 1,2,...,N \rbrace$ from the log
		\STATE Run the algorithm $\mathcal{A}$.
		\IF {the arm, selected by $\mathcal{A}$ is the arm which is shown to the user}
		\STATE Use the user feedback to update $\mathcal{A}$.
		\STATE Set $R=R+1$.
		\STATE If the user has not clicked set $L=L+1$.
		\ELSE
		\STATE Ignore this round.
		\ENDIF
		\ENDFOR
		\STATE $L$ and $R$ show the total loss and the total rounds respectively.
	\end{algorithmic} 	
\end{algorithm}
In this section, we demonstrate the superior performance of our algorithms \textit{HSB-BT} and \textit{HSB-LG} against their natural competitors \textit{EXP3} and \textit{S-EXP3} over the well known real life dataset provided by Yahoo! Research. This dataset contains a user click log for news articles displayed in the featured tab of the Today Module on Yahoo!'s front page, within October 2 to 16, 2011. The dataset contains 28041015 user visits. For each visit, the user is associated with a binary feature vector of dimension 136 that contains information about the user like age, gender, behavior targeting features, etc. We used an unbiased offline evaluation method as in \cite{eval}, to test the competitors over this dataset. A brief pseudo-code of this evaluation method is shown in Algorithm \ref{tab:table4}. In this experiment, we ran a PCA algorithm \cite{pcabook} over the first $5\%$ of the data to get the principal components of the feature vectors. We mapped the feature vectors over the first principal component to form a set of $1-$dimensional context variables. We used these context variables for \textit{S-EXP3}, \textit{HSB-BT} and \textit{HSB-LG} algorithms. We tested the \textit{EXP3} and \textit{S-EXP3} algorithms with several depth parameters, while their parameters were set to their optimum values \cite{Cesa}. However, since we do not have any information about the number of disjoint regions in the optimal mapping, i.e., $R$, the $\eta$ parameter for the \textit{HSB-BT} and \textit{HSB-LG} algorithms cannot be tuned to the optimum value analytically. In this experiment, in order to have a fair comparison, we set the $\eta$ parameter of the \textit{HSB-BT} and \textit{HSB-LG} algorithm with a specific depth equal to the $\eta$ parameter of the \textit{S-EXP3} algorithm with the same depth. We emphasize that no numerical optimization is done for the $\eta$ parameter of our algorithms. The percentage of user clicks for different algorithms are shown in Fig. \ref{yahoobar}. As shown in this table, our algorithms outperform both of the \textit{S-EXP3} and \textit{EXP3} algorithms, even though the learning rate parameters of our algorithms are not tuned to the optimum values due to the lack of knowledge on the parameter $R$.


\subsection{Real Life Classification Dataset}
In this experiment, we use well-known LandSat dataset \cite{Statlog} to show how our algorithm can be employed for online multi-class classification in the Error Correcting Output Codes (ECOC) framework \cite{ecoc}. This dataset consists of $6435$ samples from $6$ classes. The feature vectors are $36$-dimensional integer vectors.
\par In the ECOC framework, given a set of $C$ classes, we assign a binary codeword of length $N_C$ to each one of the classes. We arrange these codewords as rows of a coding matrix $M_C \in \lbrace +1,-1 \rbrace^{C \times N_C}$. We consider each one of the $N_C$ columns of $M_C$ as a binary classification problem and run a binary classifier over each column. The $i^{\text{th}}$ classifier is to learn whether the $i^{\text{th}}$ bit of the codeword is $+1$ or $-1$. In order to label a new sample, the feature vector is fed to the binary classifiers to obtain a codeword based on their outputs. We then decide on the label of the sample based on its codeword.\par 
In this experiment, we use the one-versus-all coding \cite{ecoc} to form our coding matrix as shown in table 2 and run $6$ Online Perceptrons in parallel as our binary classifiers. We use the codewords obtained from the Perceptrons as our context vectors and the classes as our bandit arms. We provide our algorithm HSB with the context vectors and label the sample based on the arm suggested by the algorithm. Then, we observe the true label and suffer a loss equal to $1$ in case of incorrect label. The competitors in this experiment are our algorithm HSB with two different hierarchical structures of "Arbitrary Position Splitting" (\textit{HSB-APS}) and "Binary Tree" (\textit{HSB-BT}), alongside \textit{EXP3}, \textit{S-EXP3} and \textit{Hamming Decoding} \cite{ecoc}. The learning parameters of the algorithms are set to their optimal value.\par 
\begin{figure}[t]
	\centering
	\includegraphics[width=0.45\textwidth]{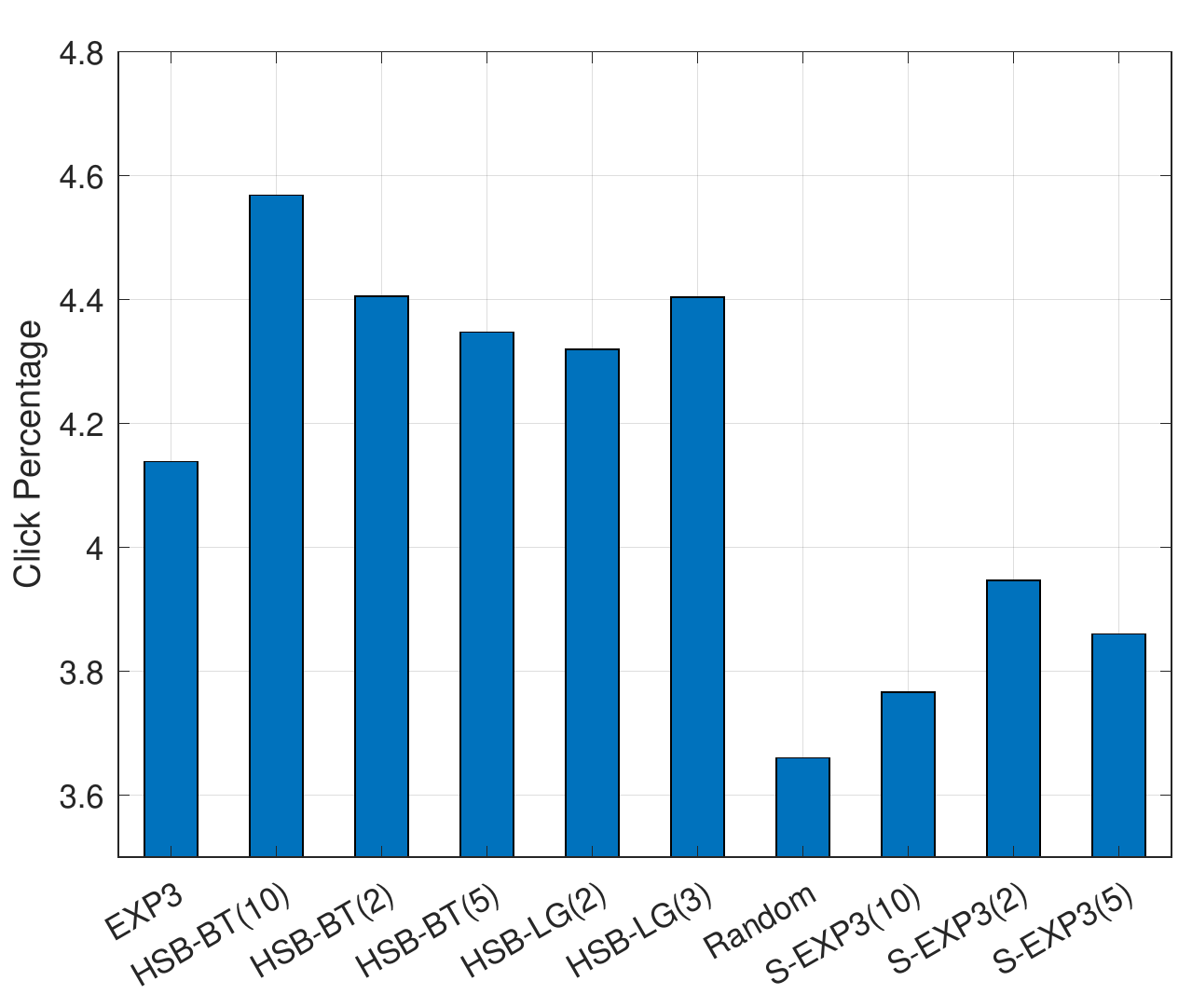}\\
	\caption{Percentage of click in the Yahoo! Today Module dataset}
	\label{yahoobar}
\end{figure}

We emphasize that while the \textit{Hamming Decoder} knows the codewords corresponding the classes a priori, other competitors do not use this information and try to learn the best mapping from the context space, i.e., codewords space, to the classes. For presentation simplicity, we have splitted the samples into $9$ consecutive epochs and averaged the number of errors over each epoch. As shown in Figure \ref{bargraph}, the algorithms \textit{S-EXP3}, \textit{HSB-BT} and \textit{HSB-APS} compensate their lack of information on the coding matrix (compared to the Hamming Decoder) as time goes on. Among them, \textit{HSB-APS} outperforms the others and even \textit{Hamming Decoder} in the last $3$ epochs as expected.

%
\section{Concluding Remarks}\label{Conc}
We studied the contextual multi-armed bandit problem in an adversarial setting and introduced truly online and low complexity algorithms that asymptotically achieve the performance of the best context dependent bandit arm selection policy. Our core algorithm quantizes the space of the context vectors into a large number of disjoint regions using an efficient quantization method and forms the class of all mappings from these regions to the bandit arms. Then, it adaptively combines these mappings in a mixture-of-experts setting and achieves the performance of the best mapping in the class. We prove performance upper bounds for the introduced algorithms. These upper bounds show that we achieve the performance of the truly optimal mapping (which might be out of our class of mappings) by increasing the number of quantization levels. We use hierarchical structures to implement our algorithms in an efficient way such that the computational complexity is log-linear in the number of quantization levels. We have no statistical assumptions on the behavior of the context vectors and the bandit arms, hence our results are guaranteed to hold in an individual sequence manner. Through extensive set of experiments involving synthetic and real data, we demonstrate the significant performance gains achieved by the proposed algorithms in comparison to the state-of-the-art techniques.
\begin{figure}[t]
	\centering
	\includegraphics[width=0.45\textwidth]{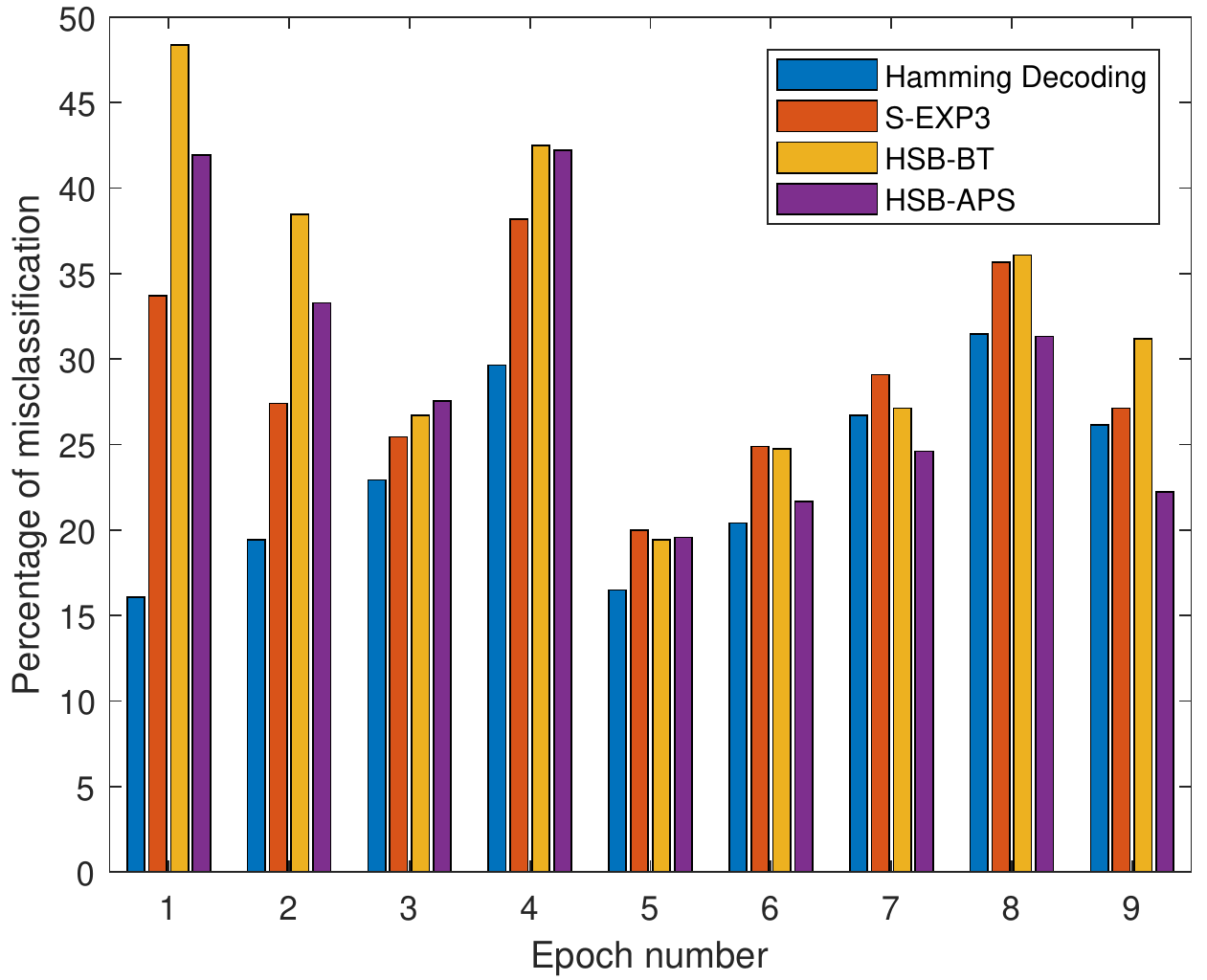}\\
	\caption{The percentage of misclassification of the competitors over $9$ consecutive epochs of length $715$.}
	\label{bargraph}
\end{figure}
\appendices
\section{Proof of Theorem 1} \label{app:1}
	From the definition, denoting the mapping followed by the $j^{\text{th}}$ expert by $g_j(.)$, we have
	\begin{equation}\label{avvalin}
	\mathcal{R}(T,E_j)=\mathbb{E}\left[\sum_{t=1}^{T}l_{t,I_t}-\sum_{t=1}^{T}l_{t,g_j({\vec{s}_t})}\right]
	\end{equation}
	Here, $l_{t,I_t}$ can be expanded as
	\begin{align}
	l_{t,I_t}&=\mathbb{E}_{j \sim \boldsymbol{\beta}_t} \tilde{l}_{t,g_j(\vec{s}_t)}\nonumber\\
	&=\frac{1}{\eta}\left(\ln \left(\mathbb{E}_{j \sim \boldsymbol{\beta}_t} e^{-\eta \tilde{l}_{t,g_j(\vec{s}_t)} }\right) + \eta \mathbb{E}_{j \sim \boldsymbol{\beta}_t} \tilde{l}_{t,g_j(\vec{s}_t)} \right)\nonumber\\
	&-\frac{1}{\eta}\ln \mathbb{E}_{j \sim \boldsymbol{\beta}_t} e^{-\eta \tilde{l}_{t,g_j(\vec{s}_t)} }. \label{twoterm}
	\end{align}
	The first term in \eqref{twoterm} can be bounded using the inequalities $\ln{x}\leq x-1$ and $\exp(-x)-1+x \leq x^2/2$, for all $x \geq 0$, as
	\begin{align}
	\ln &\left(\mathbb{E}_{j \sim \boldsymbol{\beta}_t} e^{-\eta \tilde{l}_{t,g_j(\vec{s}_t)} }\right) + \eta \mathbb{E}_{j \sim \boldsymbol{\beta}_t} \tilde{l}_{t,g_j(\vec{s}_t)}\nonumber\\
	&\leq \mathbb{E}_{j \sim \boldsymbol{\beta}_t} \left\lbrack e^{-\eta \tilde{l}_{t,g_j(\vec{s}_t)} }-1+\eta \tilde{l}_{t,g_j(\vec{s}_t)} \right\rbrack\nonumber\\
	&\leq \mathbb{E}_{j \sim \boldsymbol{\beta}_t} \frac{\eta^2 \tilde{l}_{t,g_j{(\vec{s}_t})}^2}{2}
	=\frac{\eta^2 l_{t,I_t}^2}{2 p_{t,I_t}}\leq \frac{\eta^2}{2 p_{t,I_t}}. \label{firstterm}
	\end{align}
	In order to bound the second term in \eqref{twoterm}, we just rewrite the expectation using \eqref{expweight} as follows. For $t=1$, we have
	\begin{equation}
	-\frac{1}{\eta}\ln \mathbb{E}_{j \sim \boldsymbol{\beta}_1} e^{-\eta \tilde{l}_{1,g_j(\vec{s}_1)} }=-\frac{1}{\eta} \ln \frac{\sum_{j=1}^{M^N} \alpha_{1,j} e^{-\eta \tilde{l}_{1,g_j(\vec{s}_{1})} }}{\sum_{j=1}^{M^N} \alpha_{1,j}},\label{ssecondterm}
	\end{equation}
	and for $t \geq 2$, we have
	\begin{align}
	-\frac{1}{\eta}\ln \mathbb{E}_{j \sim \boldsymbol{\beta}_t} &e^{-\eta \tilde{l}_{t,g_j(\vec{s}_t)} }\nonumber\\
	&=-\frac{1}{\eta}\ln \frac{\sum_{j=1}^{M^N} \alpha_{1,j}e^{-\eta \sum_{\tau=1}^{t} \tilde{l}_{\tau,g_j(\vec{s}_{\tau})}}}{\sum_{j=1}^{M^N} \alpha_{1,j}e^{-\eta \sum_{\tau=1}^{t-1} \tilde{l}_{\tau,g_j(\vec{s}_{\tau})}}}.\label{secondterm}
	\end{align}
	Putting the bounds in \eqref{firstterm} and \eqref{secondterm} into \eqref{twoterm}, we have
	\begin{align}
	\sum_{t=1}^{T}l_{t,I_t}	&\leq -\frac{1}{\eta} (\sum_{t=2}^{T} \ln \frac{\sum_{j=1}^{M^N} \alpha_{1,j}e^{-\eta \sum_{\tau=1}^{t} \tilde{l}_{\tau,g_j(\vec{s}_{\tau})}}}{\sum_{j=1}^{M^N} \alpha_{1,j}e^{-\eta \sum_{\tau=1}^{t-1} \tilde{l}_{\tau,g_j(\vec{s}_{\tau})}}}\nonumber\\
	&+\ln \frac{\sum_{j=1}^{M^N} \alpha_{1,j} e^{-\eta \tilde{l}_{1,g_j(\vec{s}_{1})} }}{\sum_{j=1}^{M^N} \alpha_{1,j}})+\frac{\eta T}{2 p_{t,I_t}} \label{equation14}.
	\end{align}
	Opening the first two term in \eqref{equation14}, we have
	\begin{align}
	\sum_{t=1}^{T}l_{t,I_t}	&\leq -\frac{1}{\eta} \ln \sum_{j=1}^{M^N} \alpha_{1,j}e^{-\eta \sum_{\tau=1}^{T} \tilde{l}_{\tau,g_j(\vec{s}_{\tau})}}\nonumber\\
	&+ \frac{1}{\eta} \ln \sum_{j=1}^{M^N} \alpha_{1,j}+\frac{\eta T}{2 p_{t,I_t}}\label{equation15}.
	\end{align}
	Since $\sum_{j=1}^{M^N} \alpha_{1,j}e^{-\eta \sum_{\tau=1}^{T} \tilde{l}_{\tau,g_j(\vec{s}_{\tau})}} \leq \alpha_{1,j}e^{-\eta \sum_{\tau=1}^{T} \tilde{l}_{\tau,g_j(\vec{s}_{\tau})}}$, we have
	\begin{align}
	\sum_{t=1}^{T}l_{t,I_t}	&\leq -\frac{1}{\eta} \ln \alpha_{1,j}+\sum_{\tau=1}^{T} \tilde{l}_{\tau,g_j(\vec{s}_{\tau})}\nonumber\\
	&+ \frac{1}{\eta} \ln \sum_{j=1}^{M^N} \alpha_{1,j}+\frac{\eta T}{2 p_{t,I_t}}\nonumber\\
	&=\frac{\ln 1/\beta_{1,j}}{\eta}+\frac{\eta T}{2 p_{t,I_t}} 
	+\sum_{\tau=1}^{T} \tilde{l}_{\tau,g_j(\vec{s}_{\tau})}.
	\end{align}
	Taking expectation from both sides (with respect to $I_t \sim \vec{p}_t$) and substituting $\mathbb{E}[\tilde{l}_{\tau,g_j(\vec{s}_{\tau})}]={l}_{\tau,g_j(\vec{s}_{\tau})}$ and $\mathbb{E}[\frac{1}{p_{t,I_t}}]=M$ into the result concludes the proof.
\section{Proof of Proposition 1} \label{app:2}
		We prove this proposition using induction. For leaf nodes where $\Phi_i=\emptyset$, we have
	\begin{equation}
	w_{t,i}=\frac{1}{M} \sum_{m=1}^{M} \alpha_{t,m,i}.
	\end{equation}
	From the definition of $\alpha_{t,m,i}$ in \eqref{eq:alphadef} we have
	\begin{equation}
	w_{t,i}=\sum_{m=1}^{M} \frac{1}{M}\exp(-\eta \sum_{\substack{\tau<t \\ \vec{s}_{\tau} \in r_i}}^{}\tilde{l}_{\tau,m} )=\sum_{k \in \Gamma_i}^{} \alpha_{t,k},
	\end{equation}
	where $\alpha_{1,k}=1/M$ for all $k \in \Gamma_i$.\par
	Consider the node $v_i$. Suppose $\forall \phi \in \Phi_i, \forall j \in \phi$ we have
	\begin{equation}
	w_{t,j}=\sum_{k \in \Gamma_j}^{} \alpha_{t,k}.
	\end{equation}
	It suffices to show that 
	\begin{equation}
	w_{t,i}=\sum_{k \in \Gamma_i}^{} \alpha_{t,k}.
	\end{equation}
	
	\par 
	The set of experts defined over $v_i$, i.e., $\Gamma_i$, can be decomposed into the following subsets:
	\begin{itemize}
		\item $\Gamma_i^o$ : The set of experts, which map the whole context space into a fixed arm. This set contains $M$ experts.
		\item $\Gamma_i^{\phi}$, $\phi \in \Phi_i$ : The set of experts, which partition the context space into the regions $r_j$, $j \in \phi$, and follow a specific expert over each node $j \in \phi$, based on the observed $\vec{s}_t$. If $\vec{s}_t \in r_j$, the experts in $\Gamma_i^{\phi}$ follow the experts in $\Gamma_j$. This set contains $\prod_{j \in \phi}^{} | \Gamma_j |$ experts. Each experts in $\Gamma_i^{\phi}$ can be represented by a vector of experts $\vec{k}_{\phi} \in \prod_{j \in \phi}^{}\Gamma_j$, where $\vec{k}_{\phi} (j)$ is an expert defined over node $j$.
	\end{itemize}
	We emphasize that even though we have
	\begin{equation}
	\Gamma_i^o \cup (\bigcup\limits_{\phi \in \Phi_i}^{} \Gamma_i^{\phi})=\Gamma_i,
	\end{equation}
	the intersection of any two of these $|\Phi_i|+1$ subsets is not empty necessarily. In particular, the $M$ experts in $\Gamma_i^o$ are also included among the elements of $\Gamma_i^{\phi}$ for all $\phi \in \Phi_i$. In fact, each expert in $\Gamma_i^o$ can be seen as an expert which partitions the context space into $r_j$'s for $j \in \phi$, and follows the experts which select a fixed arm $m$ over all the nodes $v_j$'s.\par 
	We have
	\begin{equation}
	\prod_{j \in \phi}^{} w_{t,j} =
	\prod_{j \in \phi}^{} \left( \sum_{k \in \Gamma_j}^{}\alpha_{t,k} \right)=\sum_{\vec{k}_{\phi} \in \prod_{j \in \phi}^{}\Gamma_j}^{} \left( \prod_{j}^{} \alpha_{t,\vec{k}_{\phi}(j)} \right).
	\end{equation}
	We open the product term as
	\begin{align} \label{eq:prodterm}
	&\prod_{j}^{} \alpha_{t,\vec{k}_{\phi}(j)}\nonumber\\
	&=\prod_{j}^{} \alpha_{1,\vec{k}_{\phi}(j)} \exp \left( -\eta \sum_{\tau < t}^{} \sum_{j}^{} \tilde{l}_{\tau,g_{\vec{k}_{\phi}(j)}(\vec{s}_{\tau})}
	\mathbf{1}_{\lbrace \vec{s}_{\tau} \in r_j \rbrace} 
	\right)\nonumber\\
	&=\prod_{j}^{} \alpha_{1,\vec{k}_{\phi}(j)} \exp 
	\left( -\eta \sum_{\substack{\tau<t \\ \vec{s}_{\tau} \in r_i}}^{} \tilde{l}_{\tau,g_{\vec{k}_{\phi}}(\vec{s}_{\tau})}
	\right).
	\end{align}
	Putting \eqref{eq:prodterm} into \eqref{eq:wrecursion} we get 
	\begin{align}
	&w_{t,i}= \frac{1}{(| \Phi_i |+1)M}\sum_{k \in \Gamma_i^o}^{}\alpha_{t,k}\nonumber\\
	+&\frac{1}{| \Phi_i |+1}\sum_{\phi \in \Phi_i}^{}\left(
	\sum_{\vec{k}_{\phi} \in \prod_{j \in \phi}^{} \Gamma_j}^{} \alpha_{1,\vec{k}_{\phi}} \exp \left( -\eta \sum_{\substack{\tau<t \\ \vec{s}_{\tau} \in r_i}}^{} \tilde{l}_{\tau,g_{\vec{k}_{\phi}(\vec{s}_{\tau})}} \right)
	\right)\nonumber \\
	=& \frac{1}{(| \Phi_i |+1)M}\sum_{k \in \Gamma_i^o}^{}\alpha_{t,k} + \frac{1}{(| \Phi_i |+1)} \sum_{\phi \in \Phi_i}^{}  \sum_{k \in \Gamma_i^{\phi}}^{} \alpha_{t,k}=\sum_{k \in \Gamma_i}^{} \alpha_{t,k},
	\end{align}
	where
	\begin{align}
	\alpha_{1,k}=& \frac{1}{(| \Phi_i |+1)M} \mathbf{1}_{\lbrace k \in \Gamma_i^o \rbrace}\nonumber\\
	+&\frac{1}{| \Phi_i |+1}
	\sum_{\phi \in \Phi_i}^{} \left(
	\mathbf{1}_{\lbrace k = \vec{k}_{\phi} \rbrace}
	\prod_{j \in \phi}^{} \alpha_{1,\vec{k}_{\phi}(j)} \right).
	\end{align}
	
	\section{Proof of Proposition 2}\label{app:3}
		Consider a specific bandit arm $m^*$. Given the context vector $\vec{s}_t$, for all $m \in \lbrace 1,2,..,M  \rbrace$, for all nodes $v_i$ in the hierarchy, we define the variables $\tilde{\alpha}_{t,m,i}$ as
	\begin{equation}
	\tilde{\alpha}_{t,m,i}=\begin{cases}
	0, &\vec{s}_t \in r_i , m \neq m^*\\
	\alpha_{t,m,i}, &\text{otherwise}
	\end{cases}
	.
	\end{equation}
	Now, from the definition of $\gamma_{t,m,i}$ in \eqref{eq:gammarecursion}, we have
	\begin{align}
	\gamma_{t,m^*,i}=&\frac{1}{(| \Phi_i |+1)M}
	\sum_{m=1}^{M} \tilde{\alpha}_{t,m,i} \nonumber \\
	+&\frac{1}{(| \Phi_i |+1)} 
	\sum_{\phi \in \Phi_i}^{}  
	\left(
	\prod_{j \in {\phi}}^{} \tilde{w}_{t,j}
	\right).
	\end{align}
	The exact same lines of the proof of Theorem \ref{th:2} hold to show that
	\begin{equation}
	\tilde{w}_{t,i}=\sum_{k \in \Gamma_i}^{} \tilde{\alpha}_{t,k},
	\end{equation}
	where
	\begin{equation}
	\tilde{\alpha}_{t,k}=\begin{cases}
	\alpha_{t,k}, & g_k(\vec{s}_t)=m^*\\
	0, &\text{otherwise}
	\end{cases}
	.
	\end{equation}
	Hence, \eqref{eq:th3} holds.\par 

\bibliographystyle{IEEEbib}
\bibliography{Ref}

\end{document}